\documentclass[11pt]{article}

\usepackage[margin=1in]{geometry}
\usepackage[T1]{fontenc}
\usepackage[utf8]{inputenc}
\usepackage{lmodern}
\usepackage[hyphens]{url}
\usepackage{xurl}
\usepackage{graphicx}
\usepackage{natbib}
\usepackage{caption}
\usepackage{booktabs}
\usepackage{amsfonts}
\usepackage{nicefrac}
\usepackage{microtype}
\usepackage{xcolor}
\usepackage{hyperref}
\hypersetup{
    colorlinks=true,
    citecolor=blue,
    linkcolor=red,
    urlcolor=blue
}
\usepackage{titling}
\thanksmarkseries{arabic}

\usepackage{amsmath}
\usepackage{amssymb,mathtools,amsthm}
\usepackage{bm}
\usepackage{float}
\usepackage{algorithm}
\usepackage{algpseudocode}

\usepackage{amsmath,amsfonts,bm}

\def\eqref#1{equation~\ref{#1}}

\def\1{\bm{1}}

\DeclareMathAlphabet{\mathsfit}{\encodingdefault}{\sfdefault}{m}{sl}
\SetMathAlphabet{\mathsfit}{bold}{\encodingdefault}{\sfdefault}{bx}{n}

\newcommand{\E}{\mathbb{E}}

\newcommand{\R}{\mathbb{R}}

\DeclareMathOperator*{\argmax}{arg\,max}
\DeclareMathOperator*{\argmin}{arg\,min}

\newtheorem{theorem}{Theorem}
\newtheorem{lemma}{Lemma}
\newtheorem{proposition}{Proposition}

\newtheorem{assumption}{Assumption}

\DeclareMathOperator{\Alt}{Alt}
\DeclareMathOperator{\dist}{dist}
\DeclareMathOperator{\kl}{kl}

\newcommand{\Prob}{\mathbb P}
\newcommand{\simp}[1]{\Delta_{#1}}
\newcommand{\ans}{A_{\preceq}}
\newcommand{\bmu}{\boldsymbol{\mu}}
\newcommand{\bnu}{\boldsymbol{\nu}}
\newcommand{\bx}{\boldsymbol{x}}
\newcommand{\free}{\mathrm{free}}
\newcommand{\cmp}{\mathrm{cmp}}
\newcommand{\coord}{\mathrm{co}}
\newcommand{\lex}{\mathrm{lex}}
\newcommand{\parord}{\mathrm{P}}

\title{When More Objectives Make Identification Easier: The Role of Preferences and Feedback}

\author{
  Bo Xue\thanks{City University of Hong Kong.
  Email: \texttt{boxue4-c@my.cityu.edu.hk}}
  \quad
  Ji Cheng\thanks{City University of Hong Kong.
  Email: \texttt{J.Cheng@my.cityu.edu.hk}}
  \quad
  Yuanyu Wan\thanks{Zhejiang University.
  Email: \texttt{wanyy@zju.edu.hk}}
  \quad
  Shuang Qiu{\color{blue}${}^\dag$}\thanks{City University of Hong Kong.
  Email: \texttt{shuanqiu@cityu.edu.hk}}
}

\usepackage[symbol]{footmisc}

\begin{document}

\date{\today}
\footnotetext[2]{Corresponding Author}

\maketitle
\begin{abstract}
Best-arm identification (BAI) aims to identify the optimal arm using as few samples as possible. While most existing work focuses on a single objective, many real-world applications involve multiple criteria, which motivates the study of BAI in multi-objective bandits. Intuitively, multi-objective identification seems harder, since the learner must identify optimal arms in a higher-dimensional reward space; however, the additional objectives also provide richer feedback that can reduce the required exploration. In this paper, we characterize when and why multiple objectives can make BAI easier. For Gaussian bandits, we establish a characteristic-time lower bound for fixed-confidence identification that holds uniformly across preference orders. Through explicit comparisons under Pareto, linear, and lexicographic preferences, we reveal three phenomena: (i) identifying a larger optimal set can require fewer samples than identifying a unique optimum; (ii) the advantage of vector feedback over coordinate feedback depends fundamentally on the preference order; and (iii) multi-objective identification can require fewer samples than even the easiest corresponding single-objective problem. Finally, we propose Order-Aware Track-and-Stop (OATS), an algorithm that is correct at any prescribed confidence level and asymptotically matches the lower bound.

\end{abstract}

\section{Introduction}
\label{sec:introduction}

Best-arm identification (BAI) is a fundamental pure-exploration problem that aims to identify the optimal arm using as few samples as possible \citep{even2006action,bubeck2009pure,garivier2016optimal}. Classical BAI considers scalar rewards and defines the optimal arm as the one with the largest mean reward \citep{audibert2010best,gabillon2012best,karnin2013almost,kaufmann2016complexity}. However, many applications involve multiple and potentially conflicting criteria, such as accuracy, cost, and latency in LLM evaluation \citep{Liang:2023-HELM,Chang:2024-survey-eval}. This motivates multi-objective BAI, in which each arm yields a reward vector and optimality is determined by a preference relation \citep{roijers2013survey,auer2016pareto,ararat2023vector,shukla2024preference}. For example, linear scalarization compares weighted rewards \citep{drugan2013designing}, lexicographic preferences rank objectives by priority \citep{huyuk2021multiobjective}, and Pareto dominance may admit multiple maximal arms \citep{auer2016pareto}. Multi-objective BAI therefore comprises distinct identification tasks whose complexity depends on the preference relation.

Multi-objective identification may appear harder than its scalar counterpart, as its complexity depends not only on reward gaps but also on the geometry of the reward vectors~\citep{karagozlu2024learning}. Under Pareto preferences, the learner may even need to identify the entire set of nondominated arms~\citep{kone2023adaptive}. On the other hand, multiple objectives also provide richer feedback: each arm pull typically reveals all reward coordinates, and evidence accumulated across objectives can reduce the required exploration~\citep{xue2026beyond}. This tension motivates our central question: \emph{When do multiple objectives make best-arm identification harder, and when can they make it easier?}

Existing studies provide algorithms and complexity guarantees for Pareto identification~\citep{auer2016pareto}, preference-cone identification~\citep{das2025frappe}, cost-aware Pareto identification~\citep{xue2026costawaremultiobjectivebanditstheory}, and lexicographic identification with cross-objective constraints~\citep{xue2026beyond}. However, these results are derived under different answer maps, model classes, and feedback models, which makes direct comparisons difficult. A unified analysis should therefore disentangle three factors: the preference order, which determines how arms are compared and hence what must be identified; the model structure, which specifies how rewards may vary across arms and objectives; and the feedback model, which determines what each sample reveals~\citep{garivier2016optimal,fiez2019sequential,degenne2020gamification,du2021combinatorial}. Their interaction, rather than the number of objectives alone, determines the sample complexity.

To address this gap, we adopt the characteristic time as a common measure of identification difficulty~\citep{garivier2016optimal,degenne2019multiple}. For Gaussian bandits, it is governed by the optimal rate at which sampling can distinguish the true instance from feasible alternatives with a different correct answer. We construct matched comparisons in which only one of the three factors (preference order, model structure, or feedback model) varies at a time. This allows us to isolate the effect of each factor on sample complexity and to quantify when, and by how much, multiple objectives facilitate identification. Our main contributions are as follows.
\begin{itemize}
\item We establish a characteristic-time lower bound for fixed-confidence multi-objective BAI under general preference-induced answer maps. This provides a unified information-theoretic basis for comparing identification tasks across preference orders, model structures, and feedback models.

\item Through explicit comparisons under Pareto and lexicographic preferences, we construct instances in which identifying multiple optimal arms is easier than identifying a unique optimum, as well as instances in which the reverse holds. Hence, neither the preference order nor the number of returned arms alone determines identification difficulty.

\item We characterize the advantage of vector feedback over coordinate feedback and establish conditions under which multi-objective identification requires fewer samples than the best corresponding single-objective benchmark. These results clarify when observations of multiple objectives accumulate into useful evidence rather than being redundant.

\item We propose Order-Aware Track-and-Stop (\textsc{OATS}) and prove its fixed-confidence correctness and asymptotic optimality. Hence, the information-theoretic gains revealed by our comparisons are achievable by an adaptive sampling algorithm.

\end{itemize}

\section{Related Work}
\label{sec:related-work}

We review two closely related lines of research: best-arm identification and multi-objective bandits.

\paragraph{Best-arm identification.}
BAI is a pure-exploration problem studied in both the fixed-budget and the fixed-confidence settings. Early work developed elimination- and confidence-based methods with gap-dependent guarantees for identifying the best arm~\citep{even2006action,bubeck2009pure,audibert2010best,kalyanakrishnan2012pac,gabillon2012best,karnin2013almost,jamieson2014lil}. A complementary information-theoretic line characterizes instance difficulty through alternative models with different correct answers, which yields fundamental lower bounds and asymptotically optimal frequentist or Bayesian sampling rules~\citep{kaufmann2016complexity,garivier2016optimal,russo2016simple}. This framework has been extended to problems with multiple valid answers and general answer maps, where the optimal allocation of samples can be formulated as a sequential game~\citep{degenne2019multiple,degenne2019games}. Work on combinatorial, linear, and transductive pure exploration further exploits structural dependencies and shared observations through optimization oracles and experimental-design principles~\citep{chen2014combinatorial,soare2014best,fiez2019sequential,degenne2020gamification}.

\paragraph{Multi-objective bandits.}
Multi-objective bandits associate each arm with a reward vector and define desirable decisions through a preference relation. Early regret-minimization studies considered Pareto dominance and scalarized utilities, and developed methods that learn nondominated arms or optimize aggregate objectives~\citep{drugan2013designing,busa2017multi}. Subsequent work extended multi-objective regret analysis to contextual, adversarial, and structured linear settings, including dominant-objective and mixed Pareto--lexicographic preferences~\citep{turgay2018multi,tekin2018multi,xu2023pareto,cheng2024hierarchize}. Within lexicographic bandits, regret guarantees have been developed for Lipschitz, adversarially corrupted, and structured low-rank models~\citep{xue2025lexicographicLipschitz,IJCAI:2025:Xue,xue2026efficientonlinelexicographicgeneralized}; hierarchical preferences have also been incorporated into safe multi-objective linear bandits~\citep{xue2026safe}. A complementary pure-exploration line studies Pareto-set identification, with confidence-region, adaptive-sampling, and elimination methods for approximate, relaxed, fixed-budget, and exact fixed-confidence recovery~\citep{auer2016pareto,kone2023adaptive,kone2024fixed,crepon2024pareto}. Recent cost-aware formulations additionally study Pareto identification when vector observations have heterogeneous evaluation costs~\citep{xue2026costawaremultiobjectivebanditstheory}. More general formulations replace coordinatewise dominance with cone-induced, hierarchical, or lexicographic orders, and characterize how preference geometry and structural relations shape identification complexity~\citep{ararat2023vector,karagozlu2024learning,shukla2024preference,das2025frappe,xue2026beyond}.

Most existing studies analyze a single preference relation under a fixed feedback model. In contrast, we compare Pareto, linear, and lexicographic identification within a common characteristic-time framework and determine when multi-objective feedback reduces the required exploration.

\section{Preliminaries}
\label{sec:preliminaries}

\paragraph{Notation.}
For any integer $K\ge 1$, let $[K]=\{1,2,\ldots,K\}$ and $\simp K=\{w\in\R_+^K:\sum_{i=1}^K w_i=1\}$ denote the probability simplex. For any matrix $\bx\in\mathbb{R}^{K\times M}$ with rows $x_1,\ldots,x_K$, its squared Frobenius norm is $\|\bx\|_F^2=\sum_{i=1}^K\|x_i\|_2^2$. For any set $S\subseteq\mathbb{R}^M$, the distance from $x\in\mathbb{R}^M$ to $S$ is $\dist(x,S)=\inf_{z\in S}\|x-z\|_2$.

We consider a stochastic multi-armed bandit with $K\ge2$ arms and $M\ge1$ objectives, where larger rewards are preferred. Each arm $i\in[K]$ is associated with an unknown mean vector $\mu_i\in\R^M$, and the bandit instance is denoted by $\bmu=(\mu_1,\ldots,\mu_K)\in\R^{K\times M}$. At round $t$, the learner selects an arm $I_t$ based on the observations collected so far and observes
\[
X_t=\mu_{I_t}+\varepsilon_t,
\qquad \varepsilon_t\sim\mathcal N(0,I_M),
\]
where the noise vectors are independent across rounds. Accordingly, sample complexity counts vector observations (i.e., arm pulls) rather than individual scalar measurements.

The identification target is determined by a known preference relation $\preceq$ on $\R^M$, with strict part denoted by $\prec$. We consider three representative orders:
\begin{align*}
    x\prec_{\parord}y
    &\iff x_m\le y_m\ \text{for every }m
      \ \text{and }x_m<y_m\ \text{for some }m,
      &&\text{(Pareto)},\\
    x\prec_v y
    &\iff v^\top x<v^\top y,
      \quad v\in\R_+^M\setminus\{0\},
      &&\text{(linear scalarization)},\\
    x\prec_{\lex}y
    &\iff x_{r(x,y)}<y_{r(x,y)},
      \quad r(x,y)=\min\{m:x_m\ne y_m\},
      &&\text{(lexicographic)}.
\end{align*}
For an instance $\bmu$, the prescribed answer is the complete set of maximal arms:
\begin{equation}
    \ans(\bmu)
    =\{i\in[K]:\nexists j\in[K]\text{ such that }\mu_i\prec\mu_j\}.
    \label{eq:answer}
\end{equation}
Under linear scalarization or lexicographic preferences, this set is typically a singleton, whereas Pareto dominance may admit multiple maximal arms. Importantly, $\ans(\bmu)$ is a single set-valued answer: the learner must recover the entire set in Eq.~(\ref{eq:answer}), rather than return any one of several valid answers as in multiple-answer identification~\citep{degenne2019multiple}. More generally, our analysis applies to any prescribed answer map with a finite range.

\paragraph{Fixed-confidence identification.}
For $\delta\in(0,1/2)$, a fixed-confidence algorithm adaptively selects arms based on past observations, stops at a data-dependent time $\tau_\delta$, and outputs a recommendation $\widehat A_\delta\subseteq[K]$. It is $\delta$-correct on a model class $\Theta\subseteq\R^{K\times M}$ if, for every $\bmu\in\Theta$,
\[
    \Prob_{\bmu}(\tau_\delta<\infty)=1,
    \qquad
    \Prob_{\bmu}\!\left(\widehat A_\delta=\ans(\bmu)\right)\ge 1-\delta.
\]
Performance is measured by the expected stopping time $\E_{\bmu}[\tau_\delta]$, which we aim to minimize subject to $\delta$-correctness. The difficulty of this problem depends on how close $\bmu$ is to instances in $\Theta$ with a different answer: the closer these alternatives are, the more samples are needed to rule them out.




In classical BAI, a standard assumption is that the best arm is unique~\citep{kaufmann2016complexity,garivier2016optimal}. With finitely many arms, this implies a strictly positive gap between the best and second-best arms, so the best arm is preserved under sufficiently small perturbations of the means. We impose the analogous condition on the general, possibly set-valued, answer map $\ans$.
\begin{assumption}[Local stability]
\label{ass:local-stability}
The true instance $\bmu\in\Theta$ is locally stable with respect to the answer map: there exists $r_{\bmu}>0$ such that
\begin{equation}
\ans(\bnu)=\ans(\bmu)
\quad
\text{for every } \bnu\in\Theta
\text{ satisfying }
\|\bnu-\bmu\|_F<r_{\bmu}.
\label{eq:local-stability}
\end{equation}
\end{assumption}
Assumption~\ref{ass:local-stability} requires the answer $\ans(\bmu)$ to remain unchanged under sufficiently small feasible perturbations. Equivalently, $\bmu$ lies in the interior of the set of instances sharing its answer.


\section{A Unified Lower Bound for Identification}
\label{sec:lower-bound}

In this section, we derive a characteristic-time lower bound that holds uniformly across preference relations. We first describe the answer-changing alternatives that any algorithm must rule out. For an instance $\bmu\in\Theta$, define its set of feasible alternatives as
\begin{equation}
    \Alt_{\preceq}(\bmu)
    =\{\bnu\in\Theta:\ans(\bnu)\ne\ans(\bmu)\}.
    \label{eq:alternatives}
\end{equation}
A $\delta$-correct algorithm must distinguish $\bmu$ from every instance in $\Alt_{\preceq}(\bmu)$. Under the Gaussian observation model, each sample from arm $i$ contributes a KL divergence of $\|\mu_i-\nu_i\|_2^2/2$ between $\bmu$ and $\bnu$. Hence, under an allocation $w\in\simp K$, the information rate against $\bnu$ is
\[
    \mathcal I(\bmu,\bnu;w)
    =\frac12\sum_{i=1}^K w_i\|\mu_i-\nu_i\|_2^2.
\]
Optimizing the allocation against the worst-case feasible alternative yields the characteristic information rate, whose reciprocal is the characteristic time:
\begin{equation}
    \mathcal I_{\preceq}^{*}(\bmu)
    =\sup_{w\in\simp K}
      \inf_{\bnu\in\Alt_{\preceq}(\bmu)}
      \mathcal I(\bmu,\bnu;w),
    \qquad
    T_{\preceq}(\bmu)=\frac1{\mathcal I_{\preceq}^{*}(\bmu)}.
    \label{eq:characteristic-time}
\end{equation}
We restrict attention to instances with nonempty alternative sets. Together with local stability (Assumption~\ref{ass:local-stability}), this guarantees
$0<\mathcal I_{\preceq}^{*}(\bmu)<\infty$.

\begin{theorem}[Unified characteristic-time lower bound]
\label{thm:lower-bound}
For any $\delta\in(0,1/2)$, any algorithm that is $\delta$-correct on $\Theta$, and any $\bmu\in\Theta$,
\begin{equation}
    \E_{\bmu}[\tau_\delta]
    \ge T_{\preceq}(\bmu)\kl(1-\delta,\delta)
    \quad
    \Rightarrow
    \quad
    \liminf_{\delta\downarrow0}
    \frac{\E_{\bmu}[\tau_\delta]}{\log(1/\delta)}
    \ge T_{\preceq}(\bmu),
    \label{eq:lower-bound}
\end{equation}
where $\kl(p,q)=p\log(p/q)+(1-p)\log((1-p)/(1-q))$ is the binary relative entropy.
\end{theorem}

Theorem~\ref{thm:lower-bound} applies uniformly to all preference relations because the order enters only through the answer-changing set $\Alt_{\preceq}(\bmu)$ and the resulting characteristic time $T_{\preceq}(\bmu)$. By a standard change-of-measure argument, any $\delta$-correct algorithm must accumulate an expected divergence of at least $\kl(1-\delta,\delta)$ between $\bmu$ and every answer-changing alternative~\citep{kaufmann2016complexity,garivier2016optimal}. Under any sampling allocation, the worst-case information gained per sample is at most $\mathcal I_{\preceq}^{*}(\bmu)$. Therefore, at least $T_{\preceq}(\bmu)\kl(1-\delta,\delta)$ samples are needed in expectation, which yields Eq.~(\ref{eq:lower-bound}). The asymptotic statement follows from $\kl(1-\delta,\delta)\sim\log(1/\delta)$ as $\delta\downarrow0$.

\paragraph{Two-arm geometry.}
For $K=2$, suppose that both $\Theta$ and the answer depend only on the difference between the two arm means. Define the set of feasible differences as
\[
    \mathcal D
    =\{\mu_1-\mu_2:\bmu=(\mu_1,\mu_2)\in\Theta\}
    \subseteq\R^M.
\]
For a fixed instance $\bmu=(\mu_1,\mu_2)\in\Theta$, let
$d=\mu_1-\mu_2\in\mathcal D$. With a slight abuse of notation, define
\[
    T_{\preceq}(d)=T_{\preceq}(\bmu),
    \qquad
    \Alt_{\preceq}(d)
    =\{\nu_1-\nu_2:\bnu=(\nu_1,\nu_2)
      \in\Alt_{\preceq}(\bmu)\}
    \subseteq\mathcal D.
\]
Both quantities are well defined and invariant under a common translation of the two arm means.

\begin{lemma}[Two-arm distance formula]
\label{lem:two-arm-distance}
If $d\in\mathcal D$ and $0<\dist(d,\Alt_{\preceq}(d))<\infty$, then
\begin{equation}
    T_{\preceq}(d)
    =\frac8{\dist^2(d,\Alt_{\preceq}(d))}.
    \label{eq:distance-time}
\end{equation}
\end{lemma}
Lemma~\ref{lem:two-arm-distance} gives a geometric interpretation: identification complexity is determined by the distance from $d$ to the nearest answer-changing difference. When several objectives must be perturbed to change the answer, their contributions add up in the squared distance, reducing the characteristic time. We next use this formula to compare preference relations.

\section{Multi-Objective Identification Gains}
\label{sec:advantages}

In this section, we use the characteristic time to determine when multiple objectives facilitate identification. Specifically, we address three questions: (i) Can a larger optimal set be easier to identify? (ii) How much can vector feedback reduce the complexity of a fixed identification task? (iii) When can a multi-objective task be easier than the best corresponding single-objective benchmark?

\subsection{Larger Optimal Sets Can Be Easier to Identify}
\label{sec:larger-optimal-set}

We first compare Pareto and lexicographic identification on the same two-arm instances. The following result shows that identification difficulty is not determined by the size of the optimal set.

\begin{proposition}[A larger optimal set can be easier to identify]
\label{prop:larger-optimal-set}
Consider the unrestricted two-arm model (Appendix~\ref{app:two-arm-comparisons}). Let $M\ge4$ be even and $\Delta>0$. For $d=(\Delta\mathbf1_{M/2},-\Delta\mathbf1_{M/2})$, the Pareto-optimal set is $\{1,2\}$ and the lexicographic optimum is $\{1\}$. Nevertheless,
\begin{equation}
    T_{\parord}(d)=\frac{16}{M\Delta^2},
    \qquad
    T_{\lex}(d)=\frac8{\Delta^2}
    \quad
    \Rightarrow
    \quad
    \frac{T_{\lex}(d)}{T_{\parord}(d)}=\frac M2.
    \label{eq:pareto-easier}
\end{equation}
Conversely, for $d=(1,\ldots,1,\epsilon)$ with $0<\epsilon<1$, both preferences identify arm $1$ as optimal, but
\begin{equation}
    T_{\parord}(d)=\frac8{\epsilon^2},
    \qquad
    T_{\lex}(d)=8
    \quad
    \Rightarrow
    \quad
    \frac{T_{\lex}(d)}{T_{\parord}(d)}=\epsilon^2.
    \label{eq:pareto-harder}
\end{equation}
\end{proposition}

In the first instance, Eq.~(\ref{eq:pareto-easier}), each arm outperforms the other on half of the objectives. Under Pareto dominance, changing the answer requires removing one arm from the Pareto-optimal set, which in turn requires closing all $M/2$ coordinate gaps on which that arm is superior. Hence, the nearest Pareto alternative lies at squared distance $M\Delta^2/2$. By contrast, the lexicographic answer is determined solely by the first coordinate, so the nearest alternative lies at squared distance $\Delta^2$. Lemma~\ref{lem:two-arm-distance} therefore yields the factor-$M/2$ separation in Eq.~(\ref{eq:pareto-easier}).

The second instance, Eq.~(\ref{eq:pareto-harder}), shows that this ordering does not hold uniformly. Although both preferences identify arm~$1$ as optimal, the Pareto answer changes as soon as the smallest coordinate gap $\epsilon$ is reversed, whereas the lexicographic answer depends only on the fixed gap in the first coordinate. Consequently, $T_{\lex}(d)/T_{\parord}(d)=\epsilon^{2}$, which can be made arbitrarily small as $\epsilon\downarrow0$. Together, these two constructions show that neither preference relation is uniformly easier than the other, and that identification complexity cannot be inferred from the size of the optimal set alone.

\subsection{The Value of Feedback Depends on the Preference Order}
\label{sec:feedback}

We next compare vector and coordinate feedback for a fixed identification task. The following result shows that the gain from vector feedback can range from none to a factor of $M$.

Under vector feedback, sampling an arm reveals all $M$ objectives, whereas coordinate feedback reveals only one chosen objective of one chosen arm. For an allocation $q\in\simp{KM}$ over arm--objective pairs, define the coordinate-feedback information rate against an alternative $\bnu$ as
\[
    \mathcal I^{\coord}(\bmu,\bnu;q)
    =\frac12\sum_{i=1}^K\sum_{m=1}^M
    q_{i,m}(\mu_{i,m}-\nu_{i,m})^2.
\]
The coordinate-feedback characteristic time $T_{\preceq}^{\coord}$ is defined analogously to~(\ref{eq:characteristic-time}), with $q$ and $\mathcal I^{\coord}$ replacing $w$ and $\mathcal I$, respectively.

\begin{proposition}[The value of feedback depends on the preference order]
\label{prop:feedback-gain}
Every instance $\bmu$ with $\Alt_{\preceq}(\bmu)\ne\emptyset$ satisfies
\begin{equation}
    T_{\preceq}(\bmu)
    \le T_{\preceq}^{\coord}(\bmu)
    \le M T_{\preceq}(\bmu).
    \label{eq:feedback-gain}
\end{equation}
In the unrestricted two-arm model with $d=\Delta\mathbf1_M$ and $\Delta>0$, the lexicographic and Pareto orders attain the lower and upper endpoints, respectively:
\begin{equation}
    T_{\lex}(d)
    =T_{\parord}(d)
    =\frac8{\Delta^2},
    \qquad
    T_{\lex}^{\coord}(d)
    =\frac8{\Delta^2},
    \qquad
    T_{\parord}^{\coord}(d)
    =\frac{8M}{\Delta^2}.
    \label{eq:feedback-endpoints}
\end{equation}
\end{proposition}

Proposition~\ref{prop:feedback-gain} shows that the benefit of vector feedback depends fundamentally on the preference order, ranging from no gain to an $M$-fold reduction in characteristic time. Under Pareto dominance, every objective can change the optimal set, so coordinate feedback must spread samples across all $M$ objectives, whereas vector feedback observes them simultaneously and achieves the full factor-$M$ improvement. Under the lexicographic order, only the first objective matters, so coordinate feedback can concentrate entirely on that objective and match vector feedback. Hence, richer feedback helps only when the preference order makes multiple objectives informative for identification.

\subsection{When Multiple Objectives Outperform Single-Objective Benchmarks}
\label{sec:scalar-acceleration}

We now compare multi-objective identification with the best single-objective benchmark. Each objective, considered alone, defines a scalar BAI problem, and we ask whether multi-objective identification can require fewer samples than even the easiest of these scalar problems.

For the two-arm gap vector $d=\mu_1-\mu_2\in\mathbb R^M$, assume that $d_m\ne0$ for all $m$. Objective $m$ alone then has characteristic time $8/d_m^2$. With $\Delta_{\max}(d)=\max_{m\in[M]}|d_m|$, we define the best single-objective benchmark and the corresponding multi-objective gain as
\[
\begin{aligned}
    T_{\mathrm{best}}(d)
    =\min_{m\in[M]}\frac8{d_m^2}
    =\frac8{\Delta_{\max}^2(d)},
    \quad
    \mathsf{Gain}_{\preceq}(d)
    =\frac{T_{\mathrm{best}}(d)}{T_{\preceq}(d)}.
\end{aligned}
\]
A gain greater than one means that the multi-objective task has a smaller characteristic time than every single-objective task. Since the multi-objective and single-objective tasks may have different answers, this is a comparison between tasks, not between feedback models for the same task.

\begin{proposition}[Acceleration under three preference orders]
\label{prop:scalar-acceleration}
Let $d\in\R^M$ have no zero coordinates, and define $S_+(d)=\sum_{m:d_m>0}d_m^2$ and $S_-(d)=\sum_{m:d_m<0}d_m^2$. Under vector feedback, strict improvement over the best single-objective benchmark occurs in the following regimes:
\begin{equation}
\begin{aligned}
&\mathsf{Gain}_{\parord}(d)
=\frac{\min\{S_+(d),S_-(d)\}}{\Delta_{\max}^2(d)}>1,
&&\text{if }\min\{S_+(d),S_-(d)\}>\Delta_{\max}^2(d),\\
&\mathsf{Gain}_{v}(d)
=\frac{(v^\top d)^2}{\|v\|_2^2\Delta_{\max}^2(d)}>1,
&&\text{if }|v^\top d|>\|v\|_2\Delta_{\max}(d),\\
&\mathsf{Gain}_{\lex}(d)
=\frac{\|d\|_2^2}{\Delta_{\max}^2(d)}>1,
&&\text{if }d\in\mathcal D
=\R_{++}^M\cup(-\R_{++}^M),\ M\ge2.
\end{aligned}
\label{eq:multi-objective-gains}
\end{equation}
\end{proposition}


Proposition~\ref{prop:scalar-acceleration} shows that multi-objective acceleration arises when changing the answer requires coordinated changes across several objectives, so that their information contributions accumulate under vector feedback. Under Pareto dominance, this occurs when the two arms are incomparable and both sign groups have sufficiently large aggregate gaps, i.e., $\min\{S_+(d),S_-(d)\}>\Delta_{\max}^2(d)$. The gain is largest when the positive and negative gaps are well balanced across objectives, and it can reach $\lfloor M/2\rfloor$. Under linear scalarization, acceleration occurs when the weighted gaps reinforce one another, i.e., $|v^\top d|>\|v\|_2\Delta_{\max}(d)$. Thus, aligned gaps promote acceleration, whereas cancellation across objectives reduces or eliminates it; the maximum gain is $M$.

Under lexicographic preferences, acceleration stems from the common-ranking constraint $\mathcal D=\R_{++}^M\cup(-\R_{++}^M)$. This constraint forces every answer-changing alternative to reverse the ordering on all objectives simultaneously, so the relevant separation scales with $\|d\|_2^2$. Consequently, strict acceleration occurs for every $M\ge2$, and the gain reaches $M$ when all coordinate gaps have equal magnitude. Without this constraint, reversing the first objective alone suffices to change the lexicographic answer, and the multi-objective acceleration disappears (Appendix~\ref{app:scalar-acceleration}).

\section{Attaining the Characteristic Time}
\label{sec:algorithm}

In this section, we show that the characteristic-time lower bound in Theorem~\ref{thm:lower-bound} can be attained asymptotically. Building on the Track-and-Stop framework \citep{garivier2016optimal,degenne2019multiple}, we propose \textsc{OATS}, an order-aware procedure that repeatedly estimates the current answer, allocates samples to distinguish it from its closest alternatives, and stops once sufficient evidence has been collected. The complete procedure is given in Algorithm~\ref{alg:order-aware}.

We first introduce some notation. Let $\mathcal A = \{\ans(\bmu):\bmu\in\Theta\}$ denote the set of possible answers, and let $\Theta_A=\{\bnu\in\Theta:\ans(\bnu)=A\}$ denote the region of instances with answer $A$. For any estimate
$\widehat \bmu=(\widehat\mu_1,\ldots,\widehat\mu_K)\in\R^{K\times M}$ and allocation
$w\in\simp K$, define
\[
    g_A(\widehat \bmu,w)
    =
    \inf_{\bnu\in\Theta\setminus\Theta_A}
    \mathcal I(\widehat \bmu,\bnu;w)
    =
    \inf_{\bnu\in\Theta\setminus\Theta_A}
    \frac12\sum_{i=1}^K
    w_i\|\widehat\mu_i-\nu_i\|_2^2,
\]
which is the smallest information rate that $w$ provides against instances with a different answer.

\begin{algorithm}[t]
\caption{\textsc{OATS}: Order-Aware Track-and-Stop}
\label{alg:order-aware}
\begin{algorithmic}[1]
\Require Preference relation $\preceq$, model class $\Theta$, confidence $\delta$
\State Pull each arm once; set $t\gets K$ and $S\gets\1_K$
\While{true}
    \State Compute the empirical mean matrix $\widehat{\bmu}_t$
    \State Fit the candidate answer $\widehat A_t$ using
           Eq.~(\ref{eq:empirical-answer})
    \State Evaluate the evidence $Z_t$ against competing answers using
           Eq.~(\ref{eq:stopping-statistic})
    \If{$Z_t>\beta(t,\delta)$}
        \State \Return $\widehat A_t$
    \EndIf
    \State Compute the most informative allocation: $w_t\in\argmax_{w\in\simp K}g_{\widehat A_t}(\widehat{\bmu}_t,w)$
    \State Mix with forced exploration:
           $q_t\gets(1-t^{-1/2})w_t+t^{-1/2}\1_K/K$
    \State Update the cumulative target: $S\gets S+q_t$
    \State Select the arm with the largest deficit:
           $I_{t+1}\gets\argmax_{i\in[K]}\{S_i-N_i(t)\}$
    \State Pull $I_{t+1}$ and update the counts and empirical means
    \State $t\gets t+1$
\EndWhile
\end{algorithmic}
\end{algorithm}

Given a confidence level $\delta\in(0,1)$, \textsc{OATS} first samples every arm once. At each round $t$, it computes the empirical mean matrix $\widehat{\bmu}_t$ and selects the answer whose region is closest to this estimate:
\begin{equation}
    \widehat A_t
    \in
    \argmin_{A\in\mathcal A}
    \inf_{\bnu\in\Theta_A}
    \|\widehat{\bmu}_t-\bnu\|_F^2.
    \label{eq:empirical-answer}
\end{equation}
Let $N_i(t)=\sum_{s=1}^t\mathbb{I}\{I_s=i\}$ denote the number of pulls of arm $i\in[K]$ up to round $t$, and let $N(t)=(N_1(t),\ldots, N_K(t))$. \textsc{OATS} then evaluates the evidence in favor of $\widehat A_t$ through
\begin{equation}
    Z_t
    =
    t\times g_{\widehat A_t}
    \left(
        \widehat{\bmu}_t,
        \frac{N(t)}{t}
    \right)
    =
    \inf_{\bnu\in\Theta\setminus\Theta_{\widehat A_t}}
    \frac12\sum_{i=1}^K
    N_i(t)
    \|\widehat\mu_i(t)-\nu_i\|_2^2.
    \label{eq:stopping-statistic}
\end{equation}
Thus, $Z_t$ measures the information accumulated against the closest instance with a different answer. We use the stopping threshold
\begin{equation}
    \beta(t,\delta)
    =
    \frac {KM}{2}
    +
    \sqrt{
        KM\log\frac{2(t+1)^{K+2}}{\delta}
    }
    +
    \log\frac{2(t+1)^{K+2}}{\delta},
    \label{eq:threshold}
\end{equation}
which controls the estimation error uniformly over time and over all possible sample-count vectors.

If $Z_t>\beta(t,\delta)$, the algorithm stops and returns $\widehat A_t$. Otherwise, \textsc{OATS} computes the most informative allocation 
\begin{equation}
    w_t\in\argmax_{w\in\simp K}g_{\widehat A_t}(\widehat{\bmu}_t,w),
    \label{eq:allocation}
\end{equation}
 which maximizes the estimated information rate against alternatives to the current answer. To ensure sufficient exploration, it then sets
\[
    q_t
    =
    (1-t^{-1/2})w_t
    +
    t^{-1/2}\frac{\1_K}{K}.
\]
The uniform component vanishes asymptotically but guarantees that every arm receives $\Omega(\sqrt{t})$ observations. Next, \textsc{OATS} updates the cumulative target and selects the arm with the largest deficit:
\[
    S\gets S+q_t,
    \qquad
    I_{t+1}
    \in
    \argmax_{i\in[K]}
    \{S_i-N_i(t)\}.
\]
Since $S_i-N_i(t)$ is the sampling deficit of arm $i$ relative to its cumulative target, this tracking rule keeps the realized sample counts within a bounded distance of their targets. Finally, \textsc{OATS} pulls arm $I_{t+1}$ and updates the counts and empirical means.

The following theorem establishes the correctness and asymptotic optimality of \textsc{OATS}.

\begin{theorem}[Asymptotic optimality of \textnormal{OATS}]
\label{thm:optimality}
Suppose every instance in $\Theta$ satisfies Assumption~\ref{ass:local-stability} and has a nonempty alternative set. Then \textnormal{OATS} is $\delta$-correct, and for every $\bmu\in\Theta$,
\begin{equation}
    \lim_{\delta\downarrow0}
    \frac{\E_{\bmu}[\tau_\delta]}{\log(1/\delta)}
    =
    T_{\preceq}(\bmu).
    \label{eq:asymptotic-optimality}
\end{equation}
\end{theorem}


Theorem~\ref{thm:optimality} shows that \textsc{OATS} is not only $\delta$-correct but also asymptotically instance-optimal: for any fixed instance, its expected stopping time matches the characteristic-time lower bound of Theorem~\ref{thm:lower-bound} to first order as $\delta\downarrow0$. This extends classical Track-and-Stop for scalar BAI~\citep{garivier2016optimal} to preference-induced, possibly set-valued identification problems. Whereas previous extensions have considered multiple valid answers~\citep{degenne2019multiple} or specific multi-objective targets such as Pareto-front identification~\citep{crepon2024pareto}, \textsc{OATS} handles different preference orders and model structures through a single answer-changing formulation. Consequently, the three phenomena of Section~\ref{sec:advantages} are attained by an adaptive procedure and are not mere properties of the lower bound. In particular, whenever the preference relation or structural constraints enlarge the information distance to the nearest answer-changing alternative, \textsc{OATS} asymptotically exploits this additional information. Applying the same analysis to a single known coordinate also shows that the oracle single-objective benchmark is attainable.

\begin{figure*}[t]
  \centering
  \includegraphics[width=\textwidth]{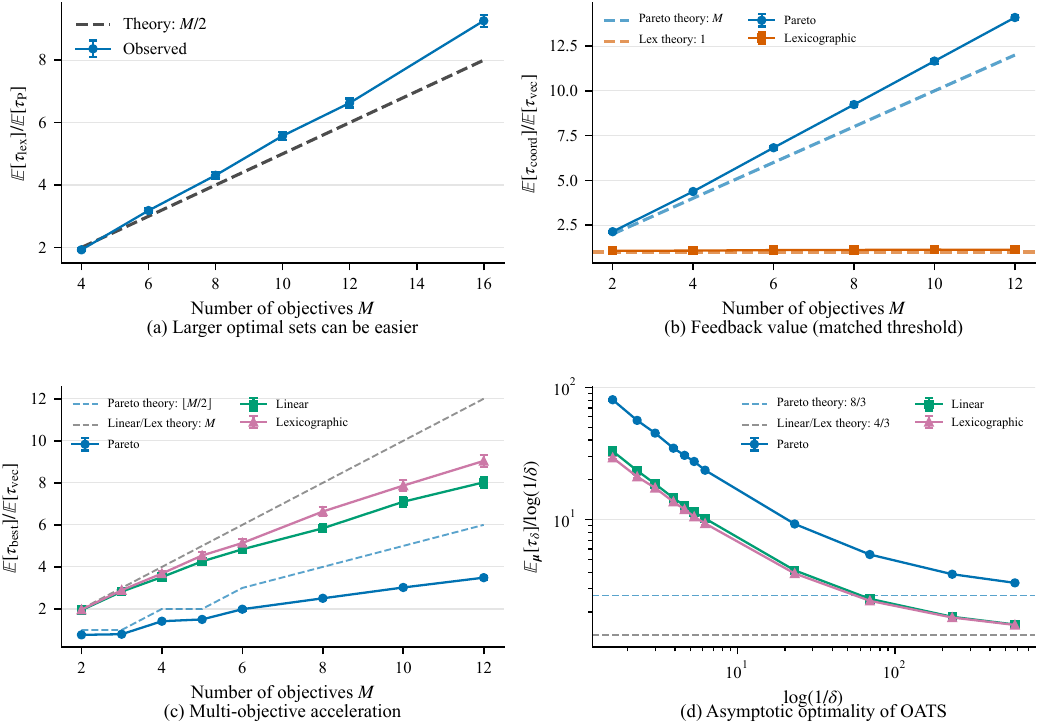}
  \caption{Synthetic results for Propositions~\ref{prop:larger-optimal-set}--\ref{prop:scalar-acceleration} and Theorem~\ref{thm:optimality}. (a) The ratio $\E[\tau_{\lex}]/\E[\tau_{\parord}]$ grows with $M$: Pareto-set identification becomes increasingly easier than lexicographic identification. (b) The feedback gain $\E[\tau_{\mathrm{coord}}]/\E[\tau_{\mathrm{vec}}]$ grows with $M$ under the Pareto order but remains near one under the lexicographic order. (c) The gain $\E[\tau_{\mathrm{best}}]/\E[\tau_{\mathrm{vec}}]$ increases with $M$ under all three orders and stays above one, confirming multi-objective acceleration. (d) The normalized stopping time $\E_{\bmu}[\tau_\delta]/\log(1/\delta)$ decreases toward the corresponding characteristic time as $\delta$ decreases. Markers show empirical means with 95\% confidence intervals, and dashed lines show theoretical predictions.}

  \label{fig:exp-synthetic}
\end{figure*}

\section{Experiments}
\label{sec:experiments}
In this section, we run synthetic experiments to test whether \textsc{OATS} exhibits the identification phenomena of Propositions~\ref{prop:larger-optimal-set}--\ref{prop:scalar-acceleration} and the asymptotic optimality of Theorem~\ref{thm:optimality}. Appendix~\ref{app:experiments} reports supplementary results on instances derived from the HELM Instruct ratings~\citep{Zhang:2024-HELM-Instruct}.

\paragraph{Experimental setup.}
All experiments use independent observations $X_{i,t}\sim\mathcal N(\mu_i,I_M)$. For two-arm instances, we set $\mu_1=d/2$ and $\mu_2=-d/2$. Unless otherwise specified, we use $\Delta=1$ and $\delta=0.05$. Each configuration is evaluated over 500 independent repetitions with deterministic seeds.

Figure~\ref{fig:exp-synthetic}(a) evaluates Proposition~\ref{prop:larger-optimal-set} for even $M\in\{4,6,8,10,12,16\}$ under the mixed-sign gap $d=(\mathbf1_{M/2},-\mathbf1_{M/2})$, for which the proposition predicts $T_{\lex}/T_{\parord}=M/2$. The empirical stopping-time ratio grows approximately linearly in $M$ and closely tracks this prediction, lying slightly above it for larger $M$. These results confirm that the advantage of Pareto-set identification over lexicographic identification scales linearly with the number of objectives.

Figure~\ref{fig:exp-synthetic}(b) validates Proposition~\ref{prop:feedback-gain} with $d=\mathbf1_M$ and $M\in\{2,4,6,8,10,12\}$, comparing vector and coordinate feedback under the Pareto and lexicographic orders. The proposition predicts coordinate-to-vector characteristic-time ratios of $M$ and $1$, respectively. The observed Pareto ratio grows approximately linearly in $M$, whereas the lexicographic ratio stays close to one. Both trends closely follow the predictions, up to a modest finite-confidence gap. These results confirm that the benefit of vector feedback depends critically on the preference order.

Figure~\ref{fig:exp-synthetic}(c) tests Proposition~\ref{prop:scalar-acceleration} for $M\in\{2,3,4,5,6,8,10,12\}$ by comparing vector-feedback identification with a scalar oracle that knows $m^*\in\arg\max_m|d_m|$. We use $d=(\mathbf1_{M/2},-\mathbf1_{M/2})$ for the Pareto order and $d=\mathbf1_M$ for the linear and common-ranking lexicographic orders, with $v=\mathbf1_M/\sqrt M$ for linear scalarization. The stopping-time gain $\mathbb E[\tau_{\mathrm{best}}]/\mathbb E[\tau_{\mathrm{vec}}]$ grows nearly linearly in $M$ for the linear and lexicographic orders and more slowly for the Pareto order, consistent with the theory. At finite confidence, the empirical gains remain below the characteristic-time predictions, particularly for larger $M$, because the thresholds in Eq.~(\ref{eq:threshold}) for the scalar and vector procedures use different dimensions, namely $KM=2$ and $KM=2M$, respectively.

Figure~\ref{fig:exp-synthetic}(d) evaluates the asymptotic optimality of \textsc{OATS} established in Theorem~\ref{thm:optimality}. We use $M=6$, with $d=(1,1,1,-1,-1,-1)$ for the Pareto order and $d=\mathbf1_6$ for the linear and lexicographic orders. We plot the normalized stopping time $\E_{\bmu}[\tau_\delta]/\log(1/\delta)$ for confidence levels $\delta$ ranging from $0.2$ down to $10^{-250}$. The theorem predicts that this quantity converges to $T_{\preceq}(\bmu)=8/3$ for Pareto identification and to $T_{\preceq}(\bmu)=4/3$ for both linear and lexicographic identification. Consistent with this prediction, all three curves decrease monotonically toward their respective limits as $\log(1/\delta)$ grows. The linear and lexicographic curves nearly coincide and approach their common limit of $4/3$, whereas the Pareto curve converges to the higher limit of $8/3$. The empirical results thus agree with the asymptotic characterization in Theorem~\ref{thm:optimality}.

\paragraph{Summary.}
Overall, the synthetic experiments confirm that larger optimal sets can be easier to identify, the benefit of vector feedback depends on the preference order, and multiple objectives can accelerate identification; moreover, the normalized stopping time of \textsc{OATS} approaches the characteristic time. Appendix~\ref{app:experiments} complements these synthetic results with implementation details, HELM-derived Gaussian instances, allocation baselines against uniform sampling and Oracle-Track, and additional robustness and scaling checks.

\section{Conclusion and Future Work}
\label{sec:conclusion}
In this paper, we studied when multiple objectives make fixed-confidence identification easier rather than harder. By casting different preference orders in a common characteristic-time framework, we disentangled the effects of the preference relation, the feasible model class, and the feedback model on identification complexity. Our results show that identification difficulty is governed by the geometry of answer-changing alternatives rather than by the number of objectives or the size of the optimal set alone. In particular, Pareto and lexicographic identification admit no universal hardness ordering, the benefit of vector feedback ranges from no gain to an $M$-fold improvement, and multi-objective identification can be strictly easier than the best corresponding single-objective problem. We further proposed \textsc{OATS}, which asymptotically matches the lower bound and thereby shows that these gains are achievable.

Future directions include extending the analysis beyond Gaussian observations to more general reward and feedback models, such as heavy-tailed noise or partially observed objectives, and developing scalable approximations of the optimal allocation for problems with many arms or objectives.

\newpage

\appendix

\begin{center}
    \LARGE \textbf{Appendix}
\end{center}

\vspace{0.5cm}

\section{Experimental Details and Additional Results}
\label{app:experiments}

\paragraph{Implementation and numerical checks.}
We implement \textsc{OATS} following Algorithm~\ref{alg:order-aware}, including the empirical answer, the stopping statistic, the forced-exploration mixture, and the cumulative-deficit tracking rule. For two-arm experiments under vector feedback, Lemma~\ref{lem:two-arm-distance} and its proof reduce the closest-alternative problem to~(\ref{eq:app-fixed-difference-cost}), which yields the balanced allocation $w=(1/2,1/2)$; the corresponding Pareto, linear, and lexicographic distances are derived in Proposition~\ref{prop:app-two-arm-times} and its proof. For coordinate feedback, we use the information formulas and optimal allocations derived in the proof of Proposition~\ref{prop:feedback-gain}. For linear identification with general $K$, we project each mean vector onto $v/\|v\|_2$, which reduces the allocation problem to scalar Gaussian best-arm identification. For a candidate singleton answer $A=\{i_A\}$, let $c_j=(v^\top(\mu_{i_A}-\mu_j))^2/(2\|v\|_2^2)$ for $j\neq i_A$, so that
\[
 g_A(\bmu,w)=\min_{j\ne i_A}c_j\frac{w_{i_A}w_j}{w_{i_A}+w_j}.
\]
The maximizing weights satisfy $w_j=w_{i_A}\lambda/(c_j-\lambda)$ for $j\neq i_A$, where $\lambda\in(0,\min_{j\ne i_A}c_j)$ is the unique solution of $\sum_{j\ne i_A}(\lambda/(c_j-\lambda))^2=1$, and $w_{i_A}$ is then determined by the normalization $\sum_{k}w_k=1$. Since the left-hand side is strictly increasing in $\lambda$, we solve this equation by bisection on $(0,\min_{j\ne i_A}c_j)$~\citep[Section~4.2.5]{Boyd:2004}; after 45 iterations, the bracketing interval has width at most $2^{-45}\min_{j\ne i_A}c_j$. We cross-check the resulting allocation and information value against a constrained epigraph formulation solved by sequential quadratic programming~\citep[Chapter~18]{Nocedal:2006}, and the two solutions agree to seven decimal places on all test instances. We further verify numerically the two-arm distance formulas, the scalar-acceleration gains, the vector- and coordinate-feedback endpoints, and the predicted scaling laws, and we confirm that instances outside the common-ranking model are correctly rejected. Every simulated trial stopped before reaching the prescribed sample cap, so all reported stopping times are finite.

\paragraph{Seeds, intervals, and baselines.}
For reproducibility, each experimental condition is assigned a fixed base seed, from which independent per-trial seeds are derived. We report 95\% confidence intervals for mean stopping times as the sample mean $\pm 1.96$ standard errors, intervals for ratios of mean stopping times via the delta method for independent samples, and intervals for error rates via the exact two-sided Clopper--Pearson method. The three allocation rules (\textsc{OATS}, Uniform, and Oracle-Track) differ only in their target allocations; they share the same observations, candidate answer, evidence statistic, and stopping threshold. Oracle-Track computes its target allocation from the true mean matrix; it is therefore not implementable in practice and serves only as a reference. In the two-arm vector-feedback setting, all three rules target the same allocation $(1/2,1/2)$, so any observed differences among them are due solely to Monte Carlo variability.

\paragraph{Stopping threshold and coordinate-feedback ablation.}
All simulations stop as soon as $Z_t$ exceeds the threshold
\[
\beta(t,\delta)=\frac{D}{2}+\sqrt{D\log\!\bigl(2(t+1)^{K_{\rm exp}+2}/\delta\bigr)}+\log\!\bigl(2(t+1)^{K_{\rm exp}+2}/\delta\bigr),
\]
where $K_{\rm exp}$ is the number of selectable experiments and $D$ is the total dimension of their mean parameters. The values dictated by our proofs are $(K_{\rm exp},D)=(K,KM)$ for vector feedback and $(KM,KM)$ for coordinate feedback, and we use them throughout, with one exception. In Figure~\ref{fig:exp-synthetic}(b), we also set $K_{\rm exp}=K$ for coordinate feedback, so that both feedback models use the same threshold and the comparison isolates the effect of the feedback model itself.

\subsection{HELM-Derived Model Evaluation}

\paragraph{Data and Gaussian construction.}
We use the publicly released GPT-4 evaluator ratings from CRFM HELM Instruct v1.0.0~\citep{Zhang:2024-HELM-Instruct}. In the HELM-derived Gaussian instances, each arm corresponds to a language model and each reward vector contains $M=5$ objectives. Pairwise comparisons use $K=2$ arms, and the linear experiments use all $K=4$ models. HELM groups prompts into evaluation scenarios, each corresponding to a fixed prompt source or task family. For each scenario $s$, model $i$, and objective $m$, we let $\widehat\mu_{i,m}^{(s)}$ denote the average released rating on objective $m$ over the prompts shared by all four models in that scenario, and we collect these five averages into the mean vector $\widehat\mu_i^{(s)}\in\mathbb R^5$. We also construct a pooled instance by averaging each model's ratings over all 678 prompts retained across the seven scenarios.

To place these means on the scale of the unit-covariance Gaussian model assumed in our theory, we compute all six model-pair Euclidean gaps within each of the seven scenarios and within the pooled instance, giving 48 gaps in total. Let $d_{\rm med}$ denote the median of the nonzero gaps. For the released ratings,
\[
    d_{\rm med}
    =\operatorname{median}_{s,\,i<j:\,\|\widehat\mu_i^{(s)}-\widehat\mu_j^{(s)}\|_2>0}
      \|\widehat\mu_i^{(s)}-\widehat\mu_j^{(s)}\|_2
    \approx 0.7561.
\]
We use the single global multiplier
\[
    \widetilde\mu_i^{(s)}=\gamma\,\widehat\mu_i^{(s)},
    \qquad
    \gamma=d_{\rm med}^{-1}\approx 1.323,
\]
so that the median model-pair separation is one after rescaling. This provides a representative signal scale relative to the identity-covariance Gaussian noise, while the median is less sensitive than the mean to unusually large or small model gaps. The simulations use the unrounded data-derived value of $\gamma$. We then generate fresh observations $X_{i,t}\sim\mathcal N(\widetilde\mu_i^{(s)},I_5)$. Because $\gamma>0$ is applied uniformly to every model and objective, it preserves all Pareto, linear, and lexicographic comparisons; it changes only the absolute signal-to-noise ratio, and hence the difficulty of the Gaussian instances.

Across the seven scenarios and the pooled instance, the six model pairs yield 48 two-arm instances with $M=5$. Figure~\ref{fig:exp-helm}(a) shows the empirical distribution of the theoretical gain over the pairs to which each notion applies, namely Pareto, linear, and common-ranking lexicographic pairs. Most Pareto and linear gains fall below one, whereas the lexicographic distribution is shifted to the right. The large accelerations in our synthetic examples are therefore not typical of HELM mean gaps. Real model pairs rarely exhibit the balanced, mixed-sign gap structure that favors Pareto identification, while common-ranking lexicographic pairs can still aggregate evidence across objectives. Figure~\ref{fig:exp-helm}(b) examines whether the characteristic time predicts stopping times on HELM-derived instances. We fix $\delta=0.05$ and consider the eligible two-arm Pareto instances and the four-arm linear instances. Observed stopping times increase with $T_{\preceq}(\bmu)\log(1/\delta)$, indicating that $T_{\preceq}(\bmu)$ remains a reliable indicator of relative instance difficulty. The points lie well above the $y=x$ reference line because, at this practical confidence level, the finite-confidence threshold is conservative and its lower-order terms are not negligible.

\begin{figure*}[t]
  \centering
  \includegraphics[width=.92\textwidth]{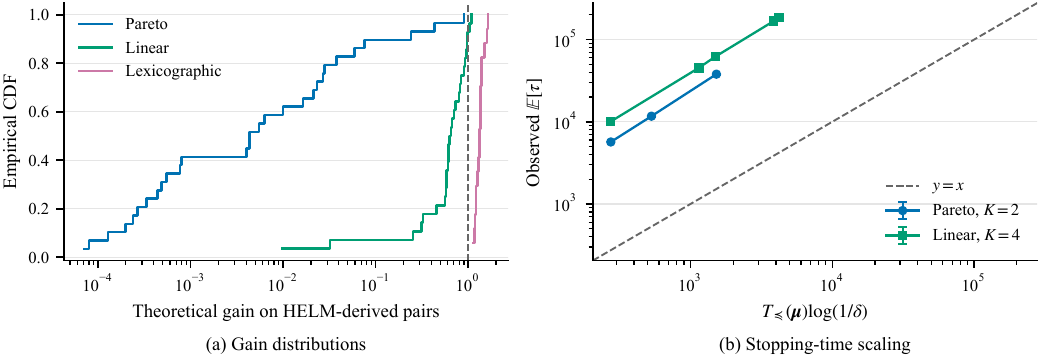}
  \caption{HELM-derived Gaussian instances. (a) Empirical CDF of the theoretical gain across within-scenario model pairs; a gain above $1$ indicates a multi-objective speedup over the best scalar benchmark. (b) Observed stopping time versus $T_{\preceq}(\bmu)\log(1/\delta)$ for eligible Pareto and linear instances, showing that $T_{\preceq}(\bmu)$ captures instance difficulty on HELM-derived geometries.}
  \label{fig:exp-helm}
\end{figure*}

\paragraph{Allocation baselines.}
We compare three allocation rules on a four-arm linear HELM instance. \textsc{OATS} estimates the answer and tracks the allocation returned by Eq.~(\ref{eq:allocation}). Uniform samples all arms equally with the same stopping rule. Oracle-Track, an unattainable reference, computes the optimal allocation from the true mean matrix with the same observations and stopping rule.

\begin{table}[H]
  \centering
  \caption{A four-arm HELM-derived linear instance (Best ChatGPT Prompts scenario, $\delta=0.05$). All methods share the same evidence statistic and stopping threshold. We report the mean stopping time, its 95\% confidence interval, and the number of incorrect answers over 500 trials. Oracle-Track uses the true mean matrix and is not implementable.}
  \label{tab:exp-helm-allocation}
  \begin{tabular}{lccc}
    \hline
    Allocation & Mean $\tau$ & 95\% CI & Errors \\
    \hline
    \textsc{OATS} & $\bm{61{,}797}$ & $[61{,}119,62{,}474]$ & $0/500$ \\
    Uniform & $87{,}600$ & $[86{,}663,88{,}538]$ & $0/500$ \\
    Oracle-Track & $61{,}938$ & $[61{,}289,62{,}588]$ & $0/500$ \\
    \hline
  \end{tabular}
\end{table}
As shown in Table~\ref{tab:exp-helm-allocation}, \textsc{OATS} reduces the mean stopping time by about $29.5\%$ relative to uniform allocation. Its performance is statistically indistinguishable from Oracle-Track, as the two 95\% confidence intervals overlap. All three methods return the correct answer in every trial.

\subsection{Additional Synthetic Experiments}

This subsection reports two additional sets of synthetic experiments. The first varies the gap size, the scalarization weights, and the balance of Pareto coordinates. The second extends linear identification to more than two arms. Unless otherwise stated, all settings follow those of the main experiments.

\begin{table*}[h]
  \centering
  \caption{Additional two-arm synthetic checks ($\delta=0.05$, 500 repetitions).}
  \label{tab:exp-additional-synthetic}
  \begin{tabular}{llcc}
    \toprule
    Check & Setting & $T_{\preceq}(\bmu)$ & Mean $\tau$ \\
    \midrule
    Gap scaling & $\Delta=0.5$ & $10.667$ & $631.2$ \\
    Gap scaling & $\Delta=1$ & $2.667$ & $133.4$ \\
    Gap scaling & $\Delta=2$ & $0.667$ & $28.3$ \\
    Unequal linear weights & $M=4$ & $2.400$ & $94.5$ \\
    Unequal linear weights & $M=8$ & $1.259$ & $58.3$ \\
    Unequal linear weights & $M=12$ & $0.855$ & $45.3$ \\
    Pareto imbalance & $p=1$ & $8.000$ & $447.1$ \\
    Pareto imbalance & $p=2$ & $4.000$ & $208.5$ \\
    Pareto imbalance & $p=3$ & $2.667$ & $131.0$ \\
    Pareto imbalance & $p=4$ & $2.000$ & $106.7$ \\
    \bottomrule
  \end{tabular}
\end{table*}

Table~\ref{tab:exp-additional-synthetic} confirms that empirical stopping times respond to the instance geometry in the same direction as the characteristic time. First, each doubling of the gap scale $\Delta$ divides $T_{\preceq}(\bmu)$ by four, consistent with its inverse-square dependence on the gap. The mean stopping time decreases by a similar factor of about $4.7$, with the difference attributable to lower-order terms in the threshold. Second, with unequal linear weights, increasing the number of objectives enlarges the projected gap and hence reduces both the characteristic time and the stopping time. Third, for Pareto identification with $M=8$ objectives split into sign groups of sizes $p$ and $8-p$, making the two groups more balanced increases the distance to the closest alternative. As a result, $T_{\preceq}(\bmu)$ scales as $1/\min\{p,8-p\}$ and the stopping time decreases accordingly.

\begin{table*}[h]
  \centering
  \caption{Linear identification with larger $K$ ($\delta=0.05$, 500 repetitions). The last three columns report empirical mean stopping times. Oracle-Track uses the true instance and is not implementable.}
  \label{tab:exp-larger-k}
  \begin{tabular}{lccccc}
    \toprule
    $K$ & $M$ & $T_{\preceq}(\bmu)$ & \textsc{OATS} & Uniform & Oracle-Track \\
    \midrule
    $5$ & $3$ & $3.568$ & $302$ & $397$ & $312$ \\
    $5$ & $5$ & $3.568$ & $354$ & $470$ & $367$ \\
    $5$ & $8$ & $3.568$ & $420$ & $561$ & $438$ \\
    $10$ & $3$ & $4.765$ & $792$ & $1586$ & $813$ \\
    $10$ & $5$ & $4.765$ & $933$ & $1876$ & $958$ \\
    $10$ & $8$ & $4.765$ & $1118$ & $2250$ & $1127$ \\
    $20$ & $3$ & $5.792$ & $1941$ & $6726$ & $1967$ \\
    $20$ & $5$ & $5.792$ & $2283$ & $7801$ & $2298$ \\
    $20$ & $8$ & $5.792$ & $2712$ & $9287$ & $2736$ \\
    \bottomrule
  \end{tabular}
\end{table*}

Table~\ref{tab:exp-larger-k} extends linear identification beyond two arms. For $K\in\{5,10,20\}$ and $M\in\{3,5,8\}$, we generate instances with a unique best arm after projection and a fixed projected gap between the best arm and its closest competitor, which avoids near-ties. This comparison isolates the value of adaptive allocation. Uniform uses the same stopping rule but samples all arms equally, whereas Oracle-Track tracks the true optimal allocation and is not implementable. For fixed $K$, $T_{\preceq}(\bmu)$ does not depend on $M$ because the projected gaps are held fixed. The increase in stopping time with $M$ instead comes from the threshold, whose dimension term $D=KM$ grows with $M$. \textsc{OATS} performs on par with Oracle-Track in every setting and is faster than uniform allocation. Its advantage grows with the number of arms, from roughly $1.3\times$ at $K=5$ to $2.0\times$ at $K=10$ and about $3.4\times$ at $K=20$.

\section{Proof of Theorem~\ref{thm:lower-bound}}
\label{app:lower-bound}

\begin{proof}
Fix $\bmu\in\Theta$. The claim is trivial if $\E_{\bmu}[\tau_\delta]=\infty$, so assume that this expectation is finite. For any $\bnu\in\Alt_{\preceq}(\bmu)$, consider the event
\[
    E=\{\widehat A_\delta=\ans(\bmu)\}.
\]
Since the algorithm is $\delta$-correct on the entire model class and $\ans(\bnu)\ne\ans(\bmu)$,
\[
    \Prob_{\bmu}(E)\ge1-\delta,
    \qquad
    \Prob_{\bnu}(E)\le\delta.
\]
Applying the change-of-measure inequality \citep[Lemma~1]{kaufmann2016complexity} to the stopped history gives
\begin{equation}
    \sum_{i=1}^K \E_{\bmu}[N_i(\tau_\delta)]
    \frac{\|\mu_i-\nu_i\|_2^2}{2}
    \ge \kl\bigl(\Prob_{\bmu}(E),\Prob_{\bnu}(E)\bigr)
    \ge \kl(1-\delta,\delta).
    \label{eq:app-change-of-measure}
\end{equation}
For completeness, the first inequality follows from the data-processing inequality applied to the indicator of $E$ in the stopped history. The relative entropy of that history equals the left-hand side, because the sampling rule is the same under both instances and each observation from arm $i$ contributes the Gaussian divergence $\|\mu_i-\nu_i\|_2^2/2$. This identity is justified by first stopping at $\tau_\delta\wedge n$ and then using the integrability of the stopped Gaussian log-likelihood increments when $\E_{\bmu}[\tau_\delta]<\infty$. The second inequality uses the monotonicity of the binary relative entropy $\kl(p,q)$ for $p\ge1-\delta>\delta\ge q$.

Define the expected sampling proportions as
\[
    \overline w_i
    =\frac{\E_{\bmu}[N_i(\tau_\delta)]}{\E_{\bmu}[\tau_\delta]}.
\]
The denominator is positive by Eq.~(\ref{eq:app-change-of-measure}), and $\sum_iN_i(\tau_\delta)=\tau_\delta$ implies $\overline w\in\simp K$. Taking the infimum of Eq.~(\ref{eq:app-change-of-measure}) over $\bnu\in\Alt_{\preceq}(\bmu)$ yields
\begin{align*}
    \kl(1-\delta,\delta)
    &\le \E_{\bmu}[\tau_\delta]
       \inf_{\bnu\in\Alt_{\preceq}(\bmu)}
       \mathcal I(\bmu,\bnu;\overline w)\\
    &\le \E_{\bmu}[\tau_\delta]\,
       \mathcal I_{\preceq}^{*}(\bmu).
\end{align*}
Using $T_{\preceq}(\bmu)=1/\mathcal I_{\preceq}^{*}(\bmu)$ proves Eq.~(\ref{eq:lower-bound}). Finally,
\[
    \kl(1-\delta,\delta)
    =(1-2\delta)\log\frac{1-\delta}{\delta},
    \qquad
    \frac{\kl(1-\delta,\delta)}{\log(1/\delta)}\longrightarrow1,
\]
which proves the asymptotic lower bound.
\end{proof}

\section{Proof of Lemma~\ref{lem:two-arm-distance}}
\label{app:two-arm-distance}

For a feasible-difference set $\mathcal D\subseteq\R^M$, write
\[
    \Theta(\mathcal D)
    =\{(\mu_1,\mu_2)\in\R^{2\times M}:
        \mu_1-\mu_2\in\mathcal D\}.
\]
The assumption of the lemma that the model depends only on the mean
difference means precisely that $\Theta=\Theta(\mathcal D)$; in
particular, the common location of the two means is unrestricted.

\begin{proof}
Fix an allocation $(\alpha,1-\alpha)\in\simp 2$ and an alternative
difference $z\in\Alt_{\preceq}(d)$. Since the model
$\Theta(\mathcal D)$ leaves the common location of the means unrestricted,
every pair with difference $z$ is feasible. Writing
$a=\mu_1-\nu_1$ and $b=\mu_2-\nu_2$, the constraint is $a-b=d-z$.
The identity
\[
    \alpha\|a\|_2^2+(1-\alpha)\|b\|_2^2
    =\alpha(1-\alpha)\|a-b\|_2^2
      +\|\alpha a+(1-\alpha)b\|_2^2
\]
therefore gives
\begin{equation}
    \inf_{\nu_1-\nu_2=z}
    \mathcal I(\bmu,\bnu;(\alpha,1-\alpha))
    =\frac{\alpha(1-\alpha)}2\|d-z\|_2^2.
    \label{eq:app-fixed-difference-cost}
\end{equation}
The infimum is attained at
$a=(1-\alpha)(d-z)$ and $b=-\alpha(d-z)$, including when
$\alpha\in\{0,1\}$. Consequently,
\[
    \mathcal I_{\preceq}^{*}(d)
    =\max_{\alpha\in[0,1]}
      \frac{\alpha(1-\alpha)}2
      \dist^2(d,\Alt_{\preceq}(d))
    =\frac18\dist^2(d,\Alt_{\preceq}(d)).
\]
Taking reciprocals proves Eq.~(\ref{eq:distance-time}).
\end{proof}

\section{Auxiliary Results for the Two-Arm Comparisons}
\label{app:two-arm-comparisons}

We now specify the feasible-difference sets used in the two-arm
comparisons. Define
\begin{equation}
\begin{aligned}
    \mathcal D_{\free}
    &=(\R\setminus\{0\})^M,\\
    \mathcal D_v
    &=\{z\in\mathcal D_{\free}:v^\top z\ne0\},\\
    \mathcal D_{\cmp}
    &=\R_{++}^M\cup(-\R_{++}^M).
\end{aligned}
\label{eq:app-comparison-models}
\end{equation}
The unrestricted Pareto and lexicographic comparisons use
$\Theta(\mathcal D_{\free})$, and linear scalarization uses
$\Theta(\mathcal D_v)$. Excluding ties ensures local stability
and does not change the distances to the decision boundaries computed
below. In the common-ranking comparison in
Proposition~\ref{prop:scalar-acceleration}, the feasible-difference set
$\mathcal D$ is $\mathcal D_{\cmp}$ and the model is
$\Theta(\mathcal D_{\cmp})$. This restriction is known to the learner and applies to the entire
model, i.e., to both the true instance and every feasible
alternative in Eq.~(\ref{eq:alternatives}).

\begin{proposition}[Two-arm characteristic times]
\label{prop:app-two-arm-times}
Let $d\in\mathcal D_{\free}$, and define
\[
    P(d)=\{m:d_m>0\},\qquad N(d)=\{m:d_m<0\},
\]
\[
    S_+(d)=\sum_{m\in P(d)}d_m^2,\qquad
    S_-(d)=\sum_{m\in N(d)}d_m^2.
\]
Under vector feedback, Pareto identification on
$\Theta(\mathcal D_{\free})$ has characteristic time
\begin{equation}
    T_{\parord}(d)=
    \begin{cases}
        \displaystyle\frac8{\min_{m\in[M]}d_m^2},
        & P(d)=\emptyset\text{ or }N(d)=\emptyset,\\[2mm]
        \displaystyle\frac8{\min\{S_+(d),S_-(d)\}},
        & P(d)\ne\emptyset\text{ and }N(d)\ne\emptyset.
    \end{cases}
    \label{eq:app-pareto-time}
\end{equation}
For $v\in\R_+^M\setminus\{0\}$ and $d\in\mathcal D_v$,
linear scalarization on $\Theta(\mathcal D_v)$ satisfies
\begin{equation}
    T_v(d)=\frac{8\|v\|_2^2}{(v^\top d)^2}.
    \label{eq:app-linear-time}
\end{equation}
For lexicographic identification, distinguish the two model classes by
writing
\begin{equation}
    T_{\lex}^{\free}(d)=\frac8{d_1^2}
    \quad(d\in\mathcal D_{\free}),
    \qquad
    T_{\lex}^{\cmp}(d)=\frac8{\|d\|_2^2}
    \quad(d\in\mathcal D_{\cmp}).
    \label{eq:app-lex-times}
\end{equation}
The superscripts in Eq.~(\ref{eq:app-lex-times}) indicate the model class,
not the feedback model.
\end{proposition}

\begin{proof}
By Lemma~\ref{lem:two-arm-distance}, it suffices to compute the squared
distance from $d$ to the corresponding set of feasible answer-changing
differences. We use $\dist(d,S)=\dist(d,\overline S)$, so boundary points
need not belong to the model to determine this infimum.

\paragraph{Pareto preferences.}
Suppose first that $d\in\R_{++}^M$. The prescribed answer is $\{1\}$.
Any alternative in $\mathcal D_{\free}$ must have a negative coordinate,
so its squared distance from $d$ is at least $\min_m d_m^2$.
Conversely, choose $j\in\argmin_m d_m^2$ and replace $d_j$ by
$-\eta$, leaving the other coordinates unchanged. This yields a feasible
answer-changing difference for every $\eta>0$, and letting
$\eta\downarrow0$ shows that the squared distance equals $\min_m d_m^2$.
The case $d\in-\R_{++}^M$ is symmetric.

If $P(d)$ and $N(d)$ are both nonempty, neither arm dominates the other,
and the answer is $\{1,2\}$. An alternative changes this answer if and only
if one arm dominates the other. Since no coordinate is zero in
$\mathcal D_{\free}$, this is equivalent to all coordinate differences
having the same sign. Hence
\[
    \overline{\Alt_{\parord}(d)}
    =\R_+^M\cup(-\R_+^M).
\]
Coordinatewise projection gives
\[
    \dist^2(d,\R_+^M)=S_-(d),
    \qquad
    \dist^2(d,-\R_+^M)=S_+(d).
\]
Both projections can be approached by strictly positive or strictly
negative differences, respectively. Therefore, the squared distance to the
alternative set is $\min\{S_+(d),S_-(d)\}$, which proves
Eq.~(\ref{eq:app-pareto-time}).

\paragraph{Linear scalarization.}
Write $a=v^\top d\ne0$. The closure of the answer-changing differences
is the halfspace
\[
    H_a=\{z\in\R^M:a\,v^\top z\le0\}.
\]
Indeed, differences whose scalarized gap has the opposite sign are
answer-changing, and excluding coordinate ties does not affect
the closure of this open halfspace. For any $z\in H_a$,
Cauchy--Schwarz implies
\[
    \|d-z\|_2
    \ge\frac{|v^\top(d-z)|}{\|v\|_2}
    \ge\frac{|a|}{\|v\|_2}.
\]
The projection
\[
    z^*=d-\frac{a}{\|v\|_2^2}v
\]
attains this distance to $H_a$. To approach $z^*$ by feasible
alternatives, take
$z^*-\eta\operatorname{sign}(a)v$ with $\eta>0$ and, if necessary,
apply an additional perturbation of norm $o(\eta)$ to avoid coordinate
ties. Such perturbations preserve the strictly opposite sign of
the scalarized gap. Thus,
\[
    \dist^2(d,\Alt_v(d))=\frac{(v^\top d)^2}{\|v\|_2^2},
\]
which proves Eq.~(\ref{eq:app-linear-time}).

\paragraph{Unrestricted lexicographic preferences.}
Suppose $d_1>0$. Every answer-changing difference in
$\mathcal D_{\free}$ has a negative first coordinate, so its
squared distance from $d$ is at least $d_1^2$. Replacing only $d_1$ by
$-\eta$ gives feasible alternatives approaching this lower bound as
$\eta\downarrow0$. Exchanging the arms covers $d_1<0$ and proves the
first identity in Eq.~(\ref{eq:app-lex-times}).

\paragraph{Lexicographic preferences under common ranking.}
Suppose $d\in\R_{++}^M$. The answer is $\{1\}$, and the feasible
answer-changing differences are exactly $-\R_{++}^M$. Consequently,
\[
    \inf_{z\in-\R_{++}^M}\|d-z\|_2^2
    =\sum_{m=1}^M d_m^2.
\]
For the lower bound, each summand $(d_m-z_m)^2$ is at least $d_m^2$.
For the upper bound, set $z=-\eta\mathbf1_M$ and let
$\eta\downarrow0$. The case $d\in-\R_{++}^M$ is symmetric. The second
identity in Eq.~(\ref{eq:app-lex-times}) now follows from
Lemma~\ref{lem:two-arm-distance}.
\end{proof}

\paragraph{The role of the feasible model.}
Both lexicographic identities can be evaluated at the same gap vector
with strictly same-sign coordinates, but they refer to different sets of
feasible alternatives. Merely observing that the true difference has a
common sign does not justify $T_{\lex}^{\cmp}(d)$: the learner must know
that all feasible instances obey the common-ranking restriction. This
distinction explains why the unrestricted lexicographic examples in
Propositions~\ref{prop:larger-optimal-set} and~\ref{prop:feedback-gain}
are consistent with the common-ranking acceleration in
Proposition~\ref{prop:scalar-acceleration}.

\section{Proof of Proposition~\ref{prop:larger-optimal-set}}
\label{app:larger-optimal-set}

\begin{proof}
Both orders are evaluated on $\Theta(\mathcal D_{\free})$.
For $d=(\Delta\mathbf1_{M/2},-\Delta\mathbf1_{M/2})$, the Pareto answer
is $\{1,2\}$, whereas the lexicographic answer is $\{1\}$. Since
\[
    S_+(d)=S_-(d)=\frac M2\Delta^2,
    \qquad d_1^2=\Delta^2,
\]
Proposition~\ref{prop:app-two-arm-times} gives
\[
    T_{\parord}(d)=\frac{16}{M\Delta^2},
    \qquad T_{\lex}(d)=\frac8{\Delta^2},
    \qquad \frac{T_{\lex}(d)}{T_{\parord}(d)}=\frac M2.
\]
This proves Eq.~(\ref{eq:pareto-easier}).

For $d=(1,\ldots,1,\epsilon)$ with $0<\epsilon<1$, both orders identify
arm $1$. Now $\min_m d_m^2=\epsilon^2$ and $d_1^2=1$, so the same
proposition yields
\[
    T_{\parord}(d)=\frac8{\epsilon^2},
    \qquad T_{\lex}(d)=8,
    \qquad \frac{T_{\lex}(d)}{T_{\parord}(d)}=\epsilon^2.
\]
This proves Eq.~(\ref{eq:pareto-harder}).
\end{proof}

\section{Proof of Proposition~\ref{prop:feedback-gain}}
\label{app:feedback-gain}

\begin{proof}
For coordinate feedback, write $q=(q_{i,m})\in\simp{KM}$ for the sampling proportions and define
\begin{equation}
    \mathcal I_{\preceq}^{*,\coord}(\bmu)
    =\sup_{q\in\simp{KM}}
      \inf_{\bnu\in\Alt_{\preceq}(\bmu)}
      \frac12\sum_{i=1}^K\sum_{m=1}^M
      q_{i,m}(\mu_{i,m}-\nu_{i,m})^2.
    \label{eq:app-coordinate-information}
\end{equation}
Its reciprocal is $T_{\preceq}^{\coord}(\bmu)$.

\paragraph{General feedback bounds.}
For any $q\in\simp{KM}$, define $w_i=\sum_mq_{i,m}$. Then $w\in\simp K$ and, against every feasible alternative,
\[
    \frac12\sum_{i,m}w_i(\mu_{i,m}-\nu_{i,m})^2
    \ge\frac12\sum_{i,m}q_{i,m}(\mu_{i,m}-\nu_{i,m})^2.
\]
Taking infima over alternatives and optimizing over allocations shows that the characteristic information under vector feedback is at least that under coordinate feedback. Conversely, any vector allocation $w$ induces the coordinate allocation $q_{i,m}=w_i/M$, which collects exactly a $1/M$ fraction of its information against every alternative. Thus,
\[
    \frac1M\mathcal I_{\preceq}^{*}(\bmu)
    \le\mathcal I_{\preceq}^{*,\coord}(\bmu)
    \le\mathcal I_{\preceq}^{*}(\bmu).
\]
Local stability and a nonempty alternative set ensure that the vector-feedback characteristic information is positive and finite, so taking reciprocals proves Eq.~(\ref{eq:feedback-gain}). We next verify the two endpoint examples on $\Theta(\mathcal D_{\free})$.

\paragraph{Information for one coordinate comparison.}
For two arms, let $p_m=q_{1,m}+q_{2,m}$. When $d_m>0$, the minimal information needed to reverse the comparison on coordinate $m$ is
\[
    c_m(q)
    =\inf_{a\le b}\left\{
       \frac{q_{1,m}}2(\mu_{1,m}-a)^2
       +\frac{q_{2,m}}2(\mu_{2,m}-b)^2\right\}
    =\frac{q_{1,m}q_{2,m}}{2p_m}d_m^2,
\]
where the right-hand side is defined as zero if $p_m=0$. If both weights are positive, the minimizing values coincide at $(q_{1,m}\mu_{1,m}+q_{2,m}\mu_{2,m})/p_m$. If a weight vanishes, the corresponding arm mean can be moved at no cost, which gives the same formula. Strict reversals approach this infimum. Since $q_{1,m}q_{2,m}\le p_m^2/4$,
\begin{equation}
    c_m(q)\le\frac{p_m d_m^2}{8},
    \label{eq:app-coordinate-balanced}
\end{equation}
with equality at $q_{1,m}=q_{2,m}=p_m/2$. Balancing the arms within each coordinate achieves these equalities simultaneously without changing $p$.

\paragraph{Pareto identification.}
At $d=\Delta\mathbf1_M$, the answer is $\{1\}$. Every alternative reverses at least one coordinate (in the closure), and reversing a single coordinate suffices to change the answer. Hence, for fixed $q$, the information infimum is $\min_m c_m(q)$. Using Eq.~(\ref{eq:app-coordinate-balanced}),
\[
    \mathcal I_{\parord}^{*,\coord}(d)
    =\frac{\Delta^2}{8}
      \max_{p\in\simp M}\min_m p_m
    =\frac{\Delta^2}{8M}.
\]
The last equality follows from $\min_m p_m\le1/M$, with equality at the uniform allocation. Thus, $T_{\parord}^{\coord}(d)=8M/\Delta^2$.

\paragraph{Lexicographic identification.}
An answer-changing difference must have $z_1\le0$, and a strict reversal of the first coordinate alone is feasible. The information infimum is therefore $c_1(q)$, giving
\[
    \mathcal I_{\lex}^{*,\coord}(d)
    =\frac{\Delta^2}{8}\max_{p\in\simp M}p_1
    =\frac{\Delta^2}{8}.
\]
This value is attained by assigning mass $1/2$ to each of the two first-coordinate experiments. Hence, $T_{\lex}^{\coord}(d)=8/\Delta^2$.

\paragraph{Vector feedback and endpoint attainment.}
At $d=\Delta\mathbf1_M$, Proposition~\ref{prop:app-two-arm-times}
gives $T_{\parord}(d)=T_{\lex}(d)=8/\Delta^2$ on
$\Theta(\mathcal D_{\free})$. Combining these identities with the
coordinate-feedback characteristic times proves Eq.~(\ref{eq:feedback-endpoints})
and shows that both bounds in Eq.~(\ref{eq:feedback-gain}) are attained.
\end{proof}

\section{Proof of Proposition~\ref{prop:scalar-acceleration}}
\label{app:scalar-acceleration}

\begin{proof}
We use the model-specific characteristic times from
Proposition~\ref{prop:app-two-arm-times}. In particular, the third line
of Eq.~(\ref{eq:multi-objective-gains}) uses the common-ranking model
$\Theta(\mathcal D_{\cmp})$ specified in
Eq.~(\ref{eq:app-comparison-models}). The scalar benchmark is
\[
    T_{\mathrm{best}}(d)=\frac8{\Delta_{\max}^2(d)}.
\]
Together with Lemma~\ref{lem:two-arm-distance}, this gives the identity
\[
    \mathsf{Gain}_{\preceq}(d)
    =\frac{\dist^2(d,\Alt_{\preceq}(d))}
           {\Delta_{\max}^2(d)},
\]
where the alternative set is taken in the relevant model.

\paragraph{Pareto preferences.}
The condition
$\min\{S_+(d),S_-(d)\}>\Delta_{\max}^2(d)$ implies that both
$P(d)$ and $N(d)$ are nonempty. Equation~(\ref{eq:app-pareto-time})
therefore yields
\[
    \mathsf{Gain}_{\parord}(d)
    =\frac{8/\Delta_{\max}^2(d)}
           {8/\min\{S_+(d),S_-(d)\}}
    =\frac{\min\{S_+(d),S_-(d)\}}
           {\Delta_{\max}^2(d)}>1.
\]
This proves the first line of Eq.~(\ref{eq:multi-objective-gains}).
For incomparable arms, $d_m^2\le\Delta_{\max}^2(d)$ gives
\[
    \mathsf{Gain}_{\parord}(d)
    \le\min\{|P(d)|,|N(d)|\}
    \le\left\lfloor\frac M2\right\rfloor.
\]
The upper bound is attained by
$d=\Delta(\mathbf1_{\lfloor M/2\rfloor},
-\mathbf1_{\lceil M/2\rceil})$ for any $\Delta>0$ and $M\ge2$.
For comparable arms, Eq.~(\ref{eq:app-pareto-time}) instead gives
\[
    \mathsf{Gain}_{\parord}(d)
    =\frac{\min_m d_m^2}{\Delta_{\max}^2(d)}\le1.
\]
Hence, the condition in the proposition characterizes strict Pareto
acceleration on the unrestricted comparison model, and such
acceleration requires $M\ge4$.

\paragraph{Linear scalarization.}
The condition $|v^\top d|>\|v\|_2\Delta_{\max}(d)$ implies
$v^\top d\ne0$. Equation~(\ref{eq:app-linear-time}) gives
\[
    \mathsf{Gain}_v(d)
    =\frac{8/\Delta_{\max}^2(d)}
           {8\|v\|_2^2/(v^\top d)^2}
    =\frac{(v^\top d)^2}
           {\|v\|_2^2\Delta_{\max}^2(d)}>1,
\]
proving the second line of Eq.~(\ref{eq:multi-objective-gains}).
Moreover, since $v\ge0$,
\[
    |v^\top d|\le\|v\|_1\Delta_{\max}(d),
    \qquad
    \mathsf{Gain}_v(d)
    \le\frac{\|v\|_1^2}{\|v\|_2^2}\le M.
\]
For a fixed $v$, the first bound is attained at
$d=\Delta\mathbf1_M$; if, in addition, the weights are equal and positive,
the gain equals $M$. Thus, the maximum gain for a fixed weight vector is
$\|v\|_1^2/\|v\|_2^2$, and the global factor-$M$ bound is attainable.

\paragraph{Lexicographic preferences under common ranking.}
For $d\in\mathcal D_{\cmp}$, the second identity in
Eq.~(\ref{eq:app-lex-times}) yields
\[
    \mathsf{Gain}_{\lex}(d)
    =\frac{8/\Delta_{\max}^2(d)}{8/\|d\|_2^2}
    =\frac{\|d\|_2^2}{\Delta_{\max}^2(d)}.
\]
All coordinates are nonzero, so for $M\ge2$,
\[
    1<\frac{\|d\|_2^2}{\Delta_{\max}^2(d)}\le M.
\]
The upper bound is attained when all coordinate gaps have the same
magnitude. This proves the third line of
Eq.~(\ref{eq:multi-objective-gains}) and the stated factor-$M$ gain.
By contrast, the first identity in Eq.~(\ref{eq:app-lex-times}) gives
$\mathsf{Gain}_{\lex}(d)=d_1^2/\Delta_{\max}^2(d)\le1$
on the unrestricted model. Thus, the lexicographic acceleration is a
consequence of the common-ranking constraint on feasible alternatives,
not merely of the signs of the true gaps.

Finally, the common-ranking restriction does not change the
single-coordinate benchmark: every two-arm scalar mean pair with a
nonzero gap can be extended to an instance in
$\Theta(\mathcal D_{\cmp})$. Hence, for each objective, the induced scalar
model remains unrestricted apart from the exclusion of ties.
The minimum in $T_{\mathrm{best}}(d)$ is an instance-wise oracle
benchmark; we do not claim that an unknown easiest coordinate can be
selected without additional sampling.
\end{proof}

\section{Proof of Theorem~\ref{thm:optimality}}
\label{app:optimality}
\label{app:continuity}

\begin{proof}
Fix $\bmu\in\Theta$, and write $A^*=\ans(\bmu)$, $g=g_{A^*}$, and $\Gamma=\mathcal I_{\preceq}^{*}(\bmu)$. We analyze the sampling process continued indefinitely, ignoring the stopping rule; for each confidence level, the actual algorithm coincides with this process up to its first threshold crossing. We use the exact allocation maximizer in Algorithm~\ref{alg:order-aware} and resolve all ties by fixed measurable rules that do not depend on $\delta$, so that the confidence level enters only through the stopping threshold. Local stability is used in the metric form of Eq.~(\ref{eq:local-stability}).

\paragraph{Step 1: continuity and positivity of the information.}
Let $S\subseteq\R^{K\times M}$ be a fixed nonempty set, and define
\[
    h(\bx,w)=\inf_{\bnu\in S}\frac12\sum_iw_i\|x_i-\nu_i\|_2^2,
    \qquad \bx\in\R^{K\times M},\quad w\in\simp K.
\]
This function is nonnegative and finite, since it is bounded above by the cost of any fixed $\bnu^0\in S$. It is concave in $w$, as an infimum of linear functions.

For fixed $\bx$, it is upper semicontinuous in $w$, as an infimum of continuous linear functions. To prove lower semicontinuity, let $w^{(n)}\to w$ and set
\[
    \eta_n=\max_{i:w_i>0}
           \left(1-\frac{w_i^{(n)}}{w_i}\right)_+.
\]
Then $\eta_n\to0$, and $w_i^{(n)}\ge(1-\eta_n)w_i$ for all $i$, including those with $w_i=0$. Since each per-arm cost is nonnegative,
\[
    h(\bx,w^{(n)})\ge(1-\eta_n)h(\bx,w).
\]
Thus, $h(\bx,\cdot)$ is continuous on the closed simplex. Moreover, the weighted seminorm $\|z\|_w=(\sum_iw_i\|z_i\|_2^2)^{1/2}$ and the triangle inequality give
\[
    \left|\sqrt{2h(\bx,w)}-\sqrt{2h(\boldsymbol{y},w)}\right|
    \le\|\bx-\boldsymbol{y}\|_w\le\|\bx-\boldsymbol{y}\|_F.
\]
Combining this uniform bound with continuity in $w$ yields joint continuity in $(\bx,w)$. Note that no closedness, convexity, or boundedness of $S$ is required. For any $A\in\mathcal A$, take $S=\Theta\setminus\Theta_A$ and write the answer-specific information as
\[
    g_A(\bx,w)
    =\inf_{\bnu\in\Theta\setminus\Theta_A}
        \frac12\sum_iw_i\|x_i-\nu_i\|_2^2.
\]
At the empirical input $\bx=\widehat{\bmu}_t$, its $i$th row is
$x_i=\widehat\mu_i(t)$. The preceding argument establishes the required
continuity and concavity of $g_A$, so its maximum over the compact
simplex is attained. Each answer-region distance is also continuous in
$\bx$, so the minimization over the finitely many answers is well defined.
Consequently, fixed measurable selectors can be used for both optimization
steps.

Local stability and the nonempty alternative set give
\[
    \rho=\inf_{\bnu\in\Theta\setminus\Theta_{A^*}}
              \|\bmu-\bnu\|_F\in(0,\infty).
\]
Evaluating the information at the uniform allocation, and at any fixed alternative $\bnu^0$, gives
\[
    0<\frac{\rho^2}{2K}\le\Gamma
      \le\frac12\max_i\|\mu_i-\nu_i^0\|_2^2<\infty.
\]

\paragraph{Step 2: a confidence event valid under adaptive sampling.}
For each arm $i$, we represent its observations by an independent infinite sequence of potential observations, and let $\widehat\mu_i[n]$ denote the mean of the first $n$ of them. For any deterministic count vector $n=(n_1,\ldots,n_K)$ with $n_i\ge1$,
\[
    Q(n)=\frac12\sum_i n_i\|\widehat\mu_i[n_i]-\mu_i\|_2^2
    \sim\frac12\chi_D^2,
    \qquad D=KM.
\]
For $0<\lambda<1$,
\[
    \log\E\exp\{\lambda(Q(n)-D/2)\}
    =\frac D2\bigl[-\log(1-\lambda)-\lambda\bigr]
    \le\frac{D\lambda^2}{4(1-\lambda)}.
\]
Chernoff's inequality with $\lambda=2\sqrt{x}/(\sqrt D+2\sqrt{x})$ then gives
\begin{equation}
    \Prob\left(Q(n)>\frac D2+\sqrt{Dx}+x\right)\le e^{-x},
    \qquad x>0.
    \label{eq:app-gaussian-tail}
\end{equation}
At time $t$, there are at most $(t+1)^K$ possible count vectors. Applying Eq.~(\ref{eq:app-gaussian-tail}) with $x=\log(2(t+1)^{K+2}/\delta)$ to each of them, before the adaptive count vector is selected, and taking a union bound over count vectors and times gives
\begin{align}
    &\Prob_{\bmu}\left(\exists t\ge K:
       \frac12\sum_iN_i(t)\|\widehat\mu_i(t)-\mu_i\|_2^2
       >\beta(t,\delta)\right)\nonumber\\
    &\qquad\le\frac{\delta}{2}\sum_{t\ge K}(t+1)^{-2}
       \le\delta.
    \label{eq:app-uniform-confidence}
\end{align}
This argument does not require the empirical means to remain Gaussian after conditioning on the adaptively chosen counts.

If $\widehat A_t\ne A^*$, the true instance belongs to $\Theta\setminus\Theta_{\widehat A_t}$, so
\[
    Z_t\le\frac12\sum_iN_i(t)\|\widehat\mu_i(t)-\mu_i\|_2^2.
\]
Therefore, an incorrect recommendation at a threshold crossing can occur only outside the confidence event in Eq.~(\ref{eq:app-uniform-confidence}), so the probability of an incorrect stop is at most $\delta$. Almost-sure stopping is established in Step~6.

\paragraph{Step 3: bounded tracking discrepancy and forced exploration.}
Let
\[
    S_i(t)=1+\sum_{s=K}^{t-1}q_{s,i},
    \qquad b_i(t)=S_i(t)-N_i(t).
\]
We have $b_i(K)=0$ and $\sum_i b_i(t)=0$. Before the pull at time $t+1$, the deficits are $b_i(t)+q_{t,i}$ and sum to one, so the selected deficit is at least $1/K$ before it is reduced by one, while every unselected deficit can only increase. By induction, $b_i(t)>-1$ for all $i$ and $t$; since the deficits sum to zero, also $b_i(t)<K-1$. Thus,
\begin{equation}
    |N_i(t)-S_i(t)|\le K-1,
    \qquad
    \left\|\frac{N(t)}t-\frac{S(t)}t\right\|_1
    \le\frac{K(K-1)}t.
    \label{eq:app-tracking}
\end{equation}
The exploration mixture satisfies $q_{s,i}\ge s^{-1/2}/K$, whence
\begin{equation}
    N_i(t)\ge1+\frac1K\sum_{s=K}^{t-1}s^{-1/2}-(K-1).
    \label{eq:app-forced-counts}
\end{equation}
Consequently, there exist deterministic constants $c_0>0$ and $t_0\ge K$, depending only on $K$, such that $N_i(t)\ge c_0\sqrt t$ for all $i$ and all $t\ge t_0$.

\paragraph{Step 4: an integrable time after which the estimates remain accurate.}
Fix $r>0$. If $\|\widehat{\bmu}_t-\bmu\|_F>r$, then at least one scalar coordinate has error greater than $r/\sqrt{KM}$. For a deterministic sample size $n$, the Gaussian tail bound gives
\[
    \Prob_{\bmu}\left(|\widehat\mu_{i,m}[n]-\mu_{i,m}|>
                         \frac r{\sqrt{KM}}\right)
    \le2\exp\left(-\frac{nr^2}{2KM}\right).
\]
Using Eq.~(\ref{eq:app-forced-counts}) and a union bound over all possible counts yields constants $C_r,c_r>0$ such that
\begin{align*}
    \Prob_{\bmu}(\|\widehat{\bmu}_t-\bmu\|_F>r)
    &\le2KM\sum_{n\ge\lceil c_0\sqrt t\rceil}
             \exp\left(-\frac{nr^2}{2KM}\right)\\
    &\le C_r e^{-c_r\sqrt t},\qquad t\ge t_0.
\end{align*}
Let $L_r$ be one plus the last integer $t\ge t_0$ at which this error exceeds $r$, with $L_r=t_0$ if no such time exists. Then
\[
    \Prob_{\bmu}(L_r>u)\le\sum_{t\ge u}C_r e^{-c_r\sqrt t}
    \quad(u\ge t_0),
\]
and
\[
    \E_{\bmu}[L_r]
    \le t_0+\sum_{t\ge t_0}(t+1)C_r e^{-c_r\sqrt t}
    <\infty.
\]
For $r<\rho/3$, an estimate within distance $r$ of $\bmu$ is at distance at most $r$ from $\Theta_{A^*}$ and at distance at least $\rho-r>r$ from every other answer region. Hence, the exact answer-region fit in Eq.~(\ref{eq:empirical-answer}) selects $\widehat A_t=A^*$ for all $t\ge L_r$. This argument also covers estimates outside $\Theta$.

\paragraph{Step 5: information accumulation after an integrable time.}
Fix $\varepsilon\in(0,\Gamma)$. By joint continuity (Step~1) and
compactness of the simplex, we can choose $r\in(0,\rho/3)$ such that
\[
    \sup_{w\in\simp K}
    |g(\bx,w)-g(\bmu,w)|\le\frac{\varepsilon}{8}
    \quad\text{whenever }\|\bx-\bmu\|_F\le r.
\]
Choose a deterministic $t_1\ge t_0$ such that
$t^{-1/2}\Gamma\le\varepsilon/8$ for all $t\ge t_1$, and let
$L=\max\{L_r,t_1\}$. Then $\E_{\bmu}[L]<\infty$.
For $t\ge L$, the candidate answer is $A^*$, and the exact allocation rule gives
\begin{align*}
    g(\bmu,w_t)
    &\ge g(\widehat{\bmu}_t,w_t)-\frac{\varepsilon}{8}\\
    &=\max_{w\in\simp K}g(\widehat{\bmu}_t,w)
       -\frac{\varepsilon}{8}\\
    &\ge\Gamma-\frac{\varepsilon}{4}.
\end{align*}
Let $u_K=\mathbf1_K/K$. By concavity and nonnegativity,
\[
    g(\bmu,q_t)
    \ge(1-t^{-1/2})g(\bmu,w_t)
       +t^{-1/2}g(\bmu,u_K)
    \ge\Gamma-\frac{3\varepsilon}{8}.
\]
Since
$S(t)/t=(K/t)u_K+t^{-1}\sum_{s=K}^{t-1}q_s$, another application
of concavity gives
\[
    g\left(\bmu,\frac{S(t)}t\right)
    \ge\frac1t\sum_{s=L}^{t-1}g(\bmu,q_s)
    \ge\left(1-\frac Lt\right)
          \left(\Gamma-\frac{3\varepsilon}{8}\right).
\]
For $t\ge8\Gamma L/\varepsilon$, this is at least
$\Gamma-\varepsilon/2$. By Eq.~(\ref{eq:app-tracking}) and the uniform
continuity of $g(\bmu,\cdot)$ on the simplex, there exists a deterministic
$t_2\ge t_0$ such that
\[
    \left|g\left(\bmu,\frac{N(t)}t\right)
          -g\left(\bmu,\frac{S(t)}t\right)\right|
    \le\frac{\varepsilon}{8}
    \qquad(t\ge t_2).
\]
Define
\[
    H_\varepsilon
    =\max\left\{t_2,
       \left\lceil\frac{8\Gamma L}{\varepsilon}\right\rceil\right\}.
\]
Then $H_\varepsilon\ge L$, $\E_{\bmu}[H_\varepsilon]<\infty$, and for
all $t\ge H_\varepsilon$,
\begin{equation}
    \frac{Z_t}{t}
    =g\left(\widehat{\bmu}_t,\frac{N(t)}t\right)
    \ge\Gamma-\frac{\varepsilon}{2}
              -\frac{\varepsilon}{8}-\frac{\varepsilon}{8}
    \ge\Gamma-\varepsilon.
    \label{eq:app-eventual-information}
\end{equation}
The random time $H_\varepsilon$ does not depend on $\delta$. This argument
relies on the average information value of the allocations rather than on
the convergence of individual maximizers, and therefore does not require
a unique optimal allocation.

\paragraph{Step 6: finite expectation and the asymptotic upper bound.}
For fixed $\delta$, the threshold satisfies $\beta(t,\delta)/t\to0$, so there exists a deterministic $b_\delta\ge K$ such that $\beta(t,\delta)<(\Gamma-\varepsilon)t$ for every integer $t\ge b_\delta$. Equation~(\ref{eq:app-eventual-information}) gives
\[
    \tau_\delta\le\max\{H_\varepsilon,b_\delta\}.
\]
This proves that the algorithm stops almost surely and has a finite expected stopping time. Together with Step~2, this establishes $\delta$-correctness.

To obtain the leading constant, set $\ell_\delta=\log(1/\delta)$, fix $\zeta>0$, and define
\[
    n_\delta=
       \left\lceil\frac{(1+\zeta)\ell_\delta}{\Gamma-\varepsilon}\right\rceil.
\]
For fixed $K,M$,
\[
    \log\frac{2(n_\delta+1)^{K+2}}{\delta}
    =\ell_\delta+O(\log\ell_\delta),
    \qquad
    \beta(n_\delta,\delta)=\ell_\delta+o(\ell_\delta).
\]
Hence, $\beta(n_\delta,\delta)<(\Gamma-\varepsilon)n_\delta$ for sufficiently small $\delta$, and this inequality continues to hold for every $t\ge n_\delta$. Indeed, extending the threshold to real $t\ge K$,
\begin{align*}
    t\,\partial_t\beta(t,\delta)
    &\le(K+2)\left(
        1+\frac{\sqrt D}{
        2\sqrt{\log\frac{2(t+1)^{K+2}}{\delta}}}
        \right),\\
    \beta(t,\delta)
    &\ge\log\frac{2(t+1)^{K+2}}{\delta}
      \ge\ell_\delta.
\end{align*}
For sufficiently small $\delta$, the second quantity exceeds the first uniformly in $t$, so $\beta(t,\delta)/t$ is decreasing in $t$. Combining this with Eq.~(\ref{eq:app-eventual-information}) gives
\[
    \tau_\delta\le\max\{H_\varepsilon,n_\delta\},
    \qquad
    \E_{\bmu}[\tau_\delta]\le n_\delta+\E_{\bmu}[H_\varepsilon].
\]
It follows that
\[
    \limsup_{\delta\downarrow0}
       \frac{\E_{\bmu}[\tau_\delta]}{\log(1/\delta)}
    \le\frac{1+\zeta}{\Gamma-\varepsilon}.
\]
Letting $\zeta\downarrow0$ and then $\varepsilon\downarrow0$ yields the upper bound $1/\Gamma=T_{\preceq}(\bmu)$. Theorem~\ref{thm:lower-bound} provides the matching lower bound, which establishes Eq.~(\ref{eq:asymptotic-optimality}).

\end{proof}



\newpage

\bibliographystyle{plainnat}
\bibliography{ref}

@inproceedings{IJCAI:2025:Xue,
  title     = {Problem-dependent Regret for Lexicographic Multi-Armed Bandits with Adversarial Corruptions},
  author    = {Xue, Bo and Lin, Xi and Wan, Yuanyu and Zhang, Qingfu},
  booktitle = {Proceedings of the 34th International Joint Conference on Artificial Intelligence},
  pages     = {6776--6784},
  year      = {2025}
}

@article{xue2025lexicographicLipschitz,
  title   = {Lexicographic Lipschitz Bandits: New Algorithms and a Lower Bound},
  author  = {Xue, Bo and Cheng, Ji and Liu, Fei and Wang, Yimu and Zhang, Lijun and Zhang, Qingfu},
  journal = {Journal of Machine Learning Research},
  volume  = {26},
  number  = {223},
  pages   = {1--56},
  year    = {2025}
}

@inproceedings{xue2026safe,
  title     = {Safe Multi-Objective Linear Bandits with Hierarchical Preferences},
  author    = {Xue, Bo and He, Mengxia and Liu, Yilu and Cheng, Ji and Zhao, Zhe and Zhang, Qingfu},
  booktitle = {Proceedings of the 35th International Joint Conference on Artificial Intelligence},
  pages     = {2418--2426},
  year      = {2026}
}

@article{xue2026costawaremultiobjectivebanditstheory,
  title={Cost-Aware Multi-Objective Bandits: Theory and Application to Budgeted LLM Configuration Evaluation},
  author={Xue, Bo and Hong, Zhi and Li, Jiayi and Wan, Yuanyu and Cheng, Ji and Qiu, Shuang},
  journal={arXiv preprint arXiv:2608.04333},
  year={2026}
}

@article{xue2026efficientonlinelexicographicgeneralized,
  title={Efficient Online Lexicographic Generalized Low-Rank Matrix Bandits},
  author={Xue, Bo and Cheng, Ji and Jing, Haodong and Li, Hongzong and Qiu, Shuang},
  journal={arXiv preprint arXiv:2608.04324},
  year={2026}
}

@book{Boyd:2004,
  title={Convex optimization},
  author={Boyd, Stephen and Vandenberghe, Lieven},
  year={2004},
  publisher={Cambridge university press}
}

@book{Nocedal:2006,
  title     = {Numerical Optimization},
  author    = {Nocedal, Jorge and Wright, Stephen J.},
  edition   = {2},
  publisher = {Springer},
  year      = {2006}
}

@article{Chang:2024-survey-eval,
  title   = {A Survey on Evaluation of Large Language Models},
  author  = {Chang, Yupeng and Wang, Xu and Wang, Jindong and Wu, Yuan and Yang, Linyi and Zhu, Kaijie and Chen, Hao and Yi, Xiaoyuan and Wang, Cunxiang and Wang, Yidong and Ye, Wei and Zhang, Yue and Chang, Yi and Yu, Philip S. and Yang, Qiang and Xie, Xing},
  journal = {ACM Transactions on Intelligent Systems and Technology},
  volume  = {15},
  number  = {3},
  pages   = {1--45},
  year    = {2024}
}

@inproceedings{drugan2013designing,
  title = {Designing Multi-Objective Multi-Armed Bandits Algorithms: A Study},
  author = {Drugan, Madalina and Nowe, Ann},
  booktitle = {The 2013 International Joint Conference on Neural Networks},
  pages = {1--8},
  year = {2013}
}

@article{kaufmann2016complexity,
  title = {On the Complexity of Best-Arm Identification in Multi-Armed Bandit Models},
  author = {Kaufmann, Emilie and Capp{\'e}, Olivier and Garivier, Aur{\'e}lien},
  journal = {Journal of Machine Learning Research},
  volume = {17},
  number = {1},
  pages = {1--42},
  year = {2016}
}

@inproceedings{garivier2016optimal,
  title     = {Optimal Best Arm Identification with Fixed Confidence},
  author    = {Garivier, Aur{\'e}lien and Kaufmann, Emilie},
  booktitle = {Proceedings of the 29th Conference on Learning Theory},
  pages     = {998--1027},
  year      = {2016}
}

@inproceedings{degenne2019multiple,
  author    = {Degenne, R{\'e}my and Koolen, Wouter M.},
  title     = {Pure Exploration with Multiple Correct Answers},
  booktitle = {Advances in Neural Information Processing Systems 32},
  pages     = {14564--14573},
  year      = {2019}
}

@inproceedings{degenne2019games,
  author    = {Degenne, R{\'e}my and Koolen, Wouter M. and M{\'e}nard, Pierre},
  title     = {Non-Asymptotic Pure Exploration by Solving Games},
  booktitle = {Advances in Neural Information Processing Systems 32},
  pages     = {14465--14474},
  year      = {2019}
}

@inproceedings{auer2016pareto,
  title     = {{Pareto} Front Identification from Stochastic Bandit Feedback},
  author    = {Auer, Peter and Chiang, Chao-Kai and Ortner, Ronald and Drugan, Madalina},
  booktitle = {Proceedings of the 19th International Conference on Artificial Intelligence and Statistics},
  pages     = {939--947},
  year      = {2016}
}

@inproceedings{crepon2024pareto,
  title     = {Sequential Learning of the {Pareto} Front for Multi-Objective Bandits},
  author    = {Cr{\'e}pon, Elise and Garivier, Aur{\'e}lien and Koolen, Wouter M.},
  booktitle = {Proceedings of the 27th International Conference on Artificial Intelligence and Statistics},
  pages     = {3583--3591},
  year      = {2024}
}

@inproceedings{ararat2023vector,
  title     = {Vector Optimization with Stochastic Bandit Feedback},
  author    = {Ararat, {\c C}a{\u g}{\i}n and Tekin, Cem},
  booktitle = {Proceedings of the 26th International Conference on Artificial Intelligence and Statistics},
  pages     = {2165--2190},
  year      = {2023}
}

@inproceedings{shukla2024preference,
  author    = {Shukla, Apurv and Basu, Debabrota},
  title     = {Preference-Based Pure Exploration},
  booktitle = {Advances in Neural Information Processing Systems 37},
  pages     = {17313--17347},
  year      = {2024}
}

@inproceedings{das2025frappe,
  author    = {Das, Udvas and Shukla, Apurv and Basu, Debabrota},
  title     = {{FraPPE}: Fast and Efficient Preference-Based Pure Exploration},
  booktitle = {Advances in Neural Information Processing Systems 38},
  pages     = {10190--10236},
  year      = {2025}
}

@inproceedings{xue2026beyond,
  author    = {Xue, Bo and Wan, Yuanyu and Lu, Zhichao and Zhang, Qingfu},
  title     = {Beyond the Lower Bound: Bridging Regret Minimization and Best Arm Identification in Lexicographic Bandits},
  booktitle = {Proceedings of the 40th AAAI Conference on Artificial Intelligence},
  pages     = {27414--27422},
  year      = {2026}
}

@inproceedings{bubeck2009pure,
  title = {Pure Exploration in Multi-Armed Bandits Problems},
  author = {Bubeck, S{\'e}bastien and Munos, R{\'e}mi and Stoltz, Gilles},
  booktitle = {Algorithmic Learning Theory},
  pages = {23--37},
  year = {2009}
}

@inproceedings{audibert2010best,
  title     = {Best Arm Identification in Multi-Armed Bandits},
  author    = {Audibert, Jean-Yves and Bubeck, S{\'e}bastien and Munos, R{\'e}mi},
  booktitle = {Proceedings of the 23rd Conference on Learning Theory},
  pages     = {41--53},
  year      = {2010}
}

@inproceedings{gabillon2012best,
  author    = {Gabillon, Victor and Ghavamzadeh, Mohammad and Lazaric, Alessandro},
  title     = {Best Arm Identification: A Unified Approach to Fixed Budget and Fixed Confidence},
  booktitle = {Advances in Neural Information Processing Systems 25},
  pages     = {3212--3220},
  year      = {2012}
}

@inproceedings{karnin2013almost,
  title     = {Almost Optimal Exploration in Multi-Armed Bandits},
  author    = {Karnin, Zohar and Koren, Tomer and Somekh, Oren},
  booktitle = {Proceedings of the 30th International Conference on Machine Learning},
  pages     = {1238--1246},
  year      = {2013}
}

@article{roijers2013survey,
  title = {A Survey of Multi-Objective Sequential Decision-Making},
  author = {Roijers, Diederik M. and Vamplew, Peter and Whiteson, Shimon and Dazeley, Richard},
  journal = {Journal of Artificial Intelligence Research},
  volume = {48},
  number = {1},
  pages = {67--113},
  year = {2013}
}

@article{huyuk2021multiobjective,
  title = {Multi-Objective Multi-Armed Bandit with Lexicographically Ordered and Satisficing Objectives},
  author = {H{\"u}y{\"u}k, Alihan and Tekin, Cem},
  journal = {Machine Learning},
  volume = {110},
  pages = {1233--1266},
  year = {2021}
}

@inproceedings{kone2023adaptive,
  author    = {Kone, Cyrille and Kaufmann, Emilie and Richert, Laura},
  title     = {Adaptive Algorithms for Relaxed Pareto Set Identification},
  booktitle = {Advances in Neural Information Processing Systems 36},
  pages     = {35190--35201},
  year      = {2023}
}

@inproceedings{karagozlu2024learning,
  title     = {Learning the {P}areto Set Under Incomplete Preferences: Pure Exploration in Vector Bandits},
  author    = {Karag{\"o}zl{\"u}, Efe Mert and Y{\i}ld{\i}r{\i}m, Ya{\c{s}}ar Cahit and Ararat, {\c{C}}a{\u{g}}{\i}n and Tekin, Cem},
  booktitle = {Proceedings of the 27th International Conference on Artificial Intelligence and Statistics},
  pages     = {3070--3078},
  year      = {2024}
}

@article{Liang:2023-HELM,
  title   = {Holistic Evaluation of Language Models},
  author  = {Liang, Percy and Bommasani, Rishi and Lee, Tony and Tsipras, Dimitris and Soylu, Dilara and Yasunaga, Michihiro and Zhang, Yian and Narayanan, Deepak and Wu, Yuhuai and Kumar, Ananya and others},
  journal = {Transactions on Machine Learning Research},
  year    = {2023}
}

@misc{Zhang:2024-HELM-Instruct,
  author = {Yian Zhang and Yifan Mai and Josselin Somerville Roberts and Rishi Bommasani and Yann Dubois and Percy Liang},
  title = {HELM Instruct: A Multidimensional Instruction Following Evaluation Framework with Absolute Ratings},
  month = {February},
  year = {2024},
  url = {https://crfm.stanford.edu/2024/02/18/helm-instruct.html},
}

@inproceedings{fiez2019sequential,
  author    = {Fiez, Tanner and Jain, Lalit and Jamieson, Kevin G. and Ratliff, Lillian},
  title     = {Sequential Experimental Design for Transductive Linear Bandits},
  booktitle = {Advances in Neural Information Processing Systems 32},
  pages     = {10666--10676},
  year      = {2019}
}

@inproceedings{degenne2020gamification,
  title     = {Gamification of Pure Exploration for Linear Bandits},
  author    = {Degenne, R{\'e}my and M{\'e}nard, Pierre and Shang, Xuedong and Valko, Michal},
  booktitle = {Proceedings of the 37th International Conference on Machine Learning},
  pages     = {2432--2442},
  year      = {2020}
}

@inproceedings{du2021combinatorial,
  author    = {Du, Yihan and Kuroki, Yuko and Chen, Wei},
  title     = {Combinatorial Pure Exploration with Full-Bandit or Partial Linear Feedback},
  booktitle = {Proceedings of the 35th AAAI Conference on Artificial Intelligence},
  pages     = {7262--7270},
  year      = {2021}
}

@article{even2006action,
  title   = {Action Elimination and Stopping Conditions for the Multi-Armed Bandit and Reinforcement Learning Problems},
  author  = {Even-Dar, Eyal and Mannor, Shie and Mansour, Yishay},
  journal = {Journal of Machine Learning Research},
  volume  = {7},
  number  = {39},
  pages   = {1079--1105},
  year    = {2006}
}

@inproceedings{kalyanakrishnan2012pac,
  title     = {{PAC} Subset Selection in Stochastic Multi-Armed Bandits},
  author    = {Kalyanakrishnan, Shivaram and Tewari, Ambuj and Auer, Peter and Stone, Peter},
  booktitle = {Proceedings of the 29th International Conference on Machine Learning},
  pages     = {655--662},
  year      = {2012}
}

@inproceedings{jamieson2014lil,
  title     = {{lil' UCB}: An Optimal Exploration Algorithm for Multi-Armed Bandits},
  author    = {Jamieson, Kevin and Malloy, Matthew and Nowak, Robert and Bubeck, S{\'e}bastien},
  booktitle = {Proceedings of the 27th Conference on Learning Theory},
  pages     = {423--439},
  year      = {2014}
}

@article{russo2016simple,
  title   = {Simple Bayesian Algorithms for Best-Arm Identification},
  author  = {Russo, Daniel},
  journal = {Operations Research},
  volume  = {68},
  number  = {6},
  pages   = {1625--1647},
  year    = {2020}
}

@inproceedings{chen2014combinatorial,
  author    = {Chen, Shouyuan and Lin, Tian and King, Irwin and Lyu, Michael R. and Chen, Wei},
  title     = {Combinatorial Pure Exploration of Multi-Armed Bandits},
  booktitle = {Advances in Neural Information Processing Systems 27},
  pages     = {379--387},
  year      = {2014}
}

@inproceedings{soare2014best,
  author    = {Soare, Marta and Lazaric, Alessandro and Munos, R{\'e}mi},
  title     = {Best-Arm Identification in Linear Bandits},
  booktitle = {Advances in Neural Information Processing Systems 27},
  pages     = {828--836},
  year      = {2014}
}

@inproceedings{busa2017multi,
  title     = {Multi-objective Bandits: Optimizing the Generalized {G}ini Index},
  author    = {Busa-Fekete, R{\'o}bert and Sz{\"o}r{\'e}nyi, Bal{\'a}zs and Weng, Paul and Mannor, Shie},
  booktitle = {Proceedings of the 34th International Conference on Machine Learning},
  pages     = {625--634},
  year      = {2017}
}

@inproceedings{turgay2018multi,
  title     = {Multi-objective Contextual Bandit Problem with Similarity Information},
  author    = {Tur{\u{g}}ay, Eralp and {\"O}ner, Doruk and Tekin, Cem},
  booktitle = {Proceedings of the 21st International Conference on Artificial Intelligence and Statistics},
  pages     = {1673--1681},
  year      = {2018}
}

@article{tekin2018multi,
  title   = {Multi-objective Contextual Multi-armed Bandit with a Dominant Objective},
  author  = {Tekin, Cem and Tur{\u{g}}ay, Eralp},
  journal = {IEEE Transactions on Signal Processing},
  volume  = {66},
  number  = {14},
  pages   = {3799--3813},
  year    = {2018}
}

@inproceedings{xu2023pareto,
  title     = {{P}areto Regret Analyses in Multi-objective Multi-armed Bandit},
  author    = {Xu, Mengfan and Klabjan, Diego},
  booktitle = {Proceedings of the 40th International Conference on Machine Learning},
  pages     = {38499--38517},
  year      = {2023}
}

@inproceedings{cheng2024hierarchize,
  author    = {Cheng, Ji and Xue, Bo and Yi, Jiaxiang and Zhang, Qingfu},
  title     = {Hierarchize {P}areto Dominance in Multi-Objective Stochastic Linear Bandits},
  booktitle = {Proceedings of the 38th AAAI Conference on Artificial Intelligence},
  pages     = {11489--11497},
  year      = {2024}
}

@inproceedings{kone2024fixed,
  title     = {Bandit {P}areto Set Identification: The Fixed-Budget Setting},
  author    = {Kone, Cyrille and Kaufmann, Emilie and Richert, Laura},
  booktitle = {Proceedings of the 27th International Conference on Artificial Intelligence and Statistics},
  pages     = {2548--2556},
  year      = {2024}
}

\end{document}